\documentclass{article}

    \PassOptionsToPackage{numbers, compress}{natbib}

    \usepackage[final]{neurips_2023}

\usepackage{amssymb, amsmath,stackengine}
\usepackage[export]{adjustbox}
\usepackage{cancel}

\usepackage[utf8]{inputenc} 
\usepackage[T1]{fontenc}    
\usepackage{hyperref}       
\usepackage{url}            
\usepackage{booktabs}       
\usepackage{amsfonts}       
\usepackage{nicefrac}       
\usepackage{microtype}      
\usepackage{xcolor}         
\usepackage{amsthm}
\usepackage{graphicx}
\usepackage{graphics}
\usepackage{multirow}
\usepackage{enumitem}
\newcommand{\tri}{\trianglelefteq}
\LetLtxMacro{\oldpm}{\pm}
\renewcommand{\pm}{$\oldpm$}
\usepackage{amsthm}
\usepackage{caption}

\newcommand{\highlight}[1]{%
  \colorbox{gray!15}{$\displaystyle#1$}}

\theoremstyle{definition}
\newtheorem{definition}{Definition}

\theoremstyle{lemma}
\newtheorem{lemma}{Lemma}
\theoremstyle{proposition}
\newtheorem{proposition}{Proposition}

\definecolor{blind1_blue}{HTML}{005AB5}
\definecolor{blind1_red}{HTML}{DC3220}

\definecolor{blind2_blue}{HTML}{1A85FF}
\definecolor{blind2_violet}{HTML}{D41159}

\title{Sheaf Hypergraph Networks}

\author{%
  Iulia 
  Duta \\
  University of Cambridge \\
  \texttt{{id366}@cam.ac.uk} \\
  \And
  Giulia 
  Cassar\`{a} \\
  University of Rome, La Sapienza\\
  \texttt{giulia.cassara@uniroma1.it} \\
  \And
  Fabrizio
  Silvestri \\
  University of Rome, La Sapienza\\
  \texttt{fabrizio.silvestri@uniroma1.it} \\
  \And
  Pietro Li\`{o}\\
  University of Cambridge \\
  \texttt{pl219@cam.ac.uk} \\
}

\begin{document}

\maketitle
\begin{abstract}

Higher-order relations are widespread in nature, with numerous phenomena involving complex interactions that extend beyond simple pairwise connections. As a result, advancements in higher-order processing can accelerate the growth of various fields requiring structured data. Current approaches typically represent these interactions using hypergraphs.
We enhance this representation by introducing cellular sheaves for hypergraphs, a mathematical construction that adds extra structure to the conventional hypergraph while maintaining their local, higher-order connectivity. Drawing inspiration from existing Laplacians in the literature, we develop two unique formulations of sheaf hypergraph Laplacians: linear and non-linear. Our theoretical analysis demonstrates that incorporating sheaves into the hypergraph Laplacian provides a more expressive inductive bias than standard hypergraph diffusion, creating a powerful instrument for effectively modelling complex data structures.
We employ these sheaf hypergraph Laplacians to design two categories of models: Sheaf Hypergraph Neural Networks and Sheaf Hypergraph Convolutional Networks. These models generalize classical Hypergraph Networks often found in the literature. Through extensive experimentation, we show that this generalization significantly improves performance, achieving top results on multiple benchmark datasets for hypergraph node classification.
\end{abstract}

\section{Introduction}


        


The prevalence of relational data in real-world scenarios has led to rapid development and widespread adoption of graph-based methods in numerous domains~\cite{pred_pat_outcome,huang2020skipgnn_molecular_interactions,ying2018graph_recommended,wang2018videos_gupta2}. However, a major limitation of graphs is their inability to represent interactions that goes beyond pairwise relations. In contrast, real-world interactions are often complex and multifaceted.
There is evidence that higher-order relations frequently occur in neuroscience~\cite{Crossley2014TheHO,brain_ho},  chemistry~\cite{jost2018hypergraph_chemistry}, environmental science~\cite{environment_ho,pollinators} and social networks~\cite{social_ho}. 
Consequently, learning powerful and meaningful representations for hypergraphs has emerged as a promising and rapidly growing subfield of deep learning~\cite{HyperNN,ramon_hypergraph,hyper_altzheimer,reidentification_hyper,HEAT,hyper_traffic_pred}. 
%
However, current hypergraph-based models struggle to capture higher-order relationships effectively. As described in~\cite{oversmoothing_hyper}, conventional hypergraph neural networks often suffer from the  problem of over-smoothing. As we propagate the information inside the hypergraph, the representations of the nodes become uniform across neighbourhoods. This effect hampers the capability of hypergraph models to capture local, higher-order nuances. 

%
More powerful and flexible mathematical constructs are required to capture real-world interactions' complexity better. Sheaves provide a suitable enhancement for graphs that allow for more diverse and expressive representations. A cellular sheaf~\cite{sheaf_curry} enables to attach data to a graph, by associating vector spaces to the nodes, together with a mechanism of transferring the information along the edges. This approach allows for richer data representation and enhances the ability to model complex interactions. 

Motivated by the need for more expressive structures, we introduce a \textit{cellular sheaf for hypergraphs}, which allows for the representation of more sophisticated dynamics while preserving the higher-order connectivity inherent to hypergraphs. We take on the non-trivial challenge to generalize the two commonly used hypergraph Laplacians~\cite{nonlin_lap, HyperNN} to incorporate the richer structure sheaves offer. Theoretically, we demonstrate that the diffusion process derived using the \textit{sheaf hypergraph Laplacians} that we propose induces a more expressive inductive bias than the classical hypergraph diffusion. Leveraging this enhanced inductive bias, we construct and test two powerful neural networks capable of inferring and processing hypergraph sheaf structure: the \textit{Sheaf Hypergraph Neural Network} (SheafHyperGNN) and the \textit{Sheaf Hypergraph Convolutional Network} (SheafHyperGCN).

The introduction of the cellular sheaf for hypergraphs expands the potential for representing complex interactions and provides a foundation for more advanced techniques. By generalizing the hypergraph Laplacians with the sheaf structure, we can better capture the nuance and intricacy of real-world data. Furthermore, our theoretical analysis provides evidence that the sheaf hypergraph Laplacians embody a more expressive inductive bias, essential for obtaining strong representations.

\textbf{Our main contributions} are summarised as follow:
\begin{enumerate}
    \item We introduce the \textbf{cellular sheaf for hypergraphs}, a mathematical construct that enhances the hypergraphs with additional structure by associating a vector space with each node and hyperedge, along with linear projections that enable information transfer between them.
    \item We propose both a \textbf{linear} and a \textbf{non-linear sheaf hypergraph Laplacian}, generalizing the standard hypergraph Laplacians commonly used in the literature. We also provide a theoretical characterization of the inductive biases generated by the diffusion processes of these Laplacians, showcasing the benefits of utilizing these novel tools for effectively modeling intricate phenomena.
    \item The two sheaf hypergraph Laplacians are the foundation for \textbf{two novel architectures} tailored for hypergraph processing: \textbf{Sheaf Hypergraph Neural Network} and \textbf{Sheaf Hypergraph Convolutional Network}. Experimental findings demonstrate that these models achieve top results, surpassing existing methods on numerous benchmarking datasets. 
\end{enumerate}


\section{Related work} 

\textbf{Sheaves on Graphs.} 
Utilizing graph structure in real-world data has improved various domains like healthcare~\citep{pred_pat_outcome}, biochemistry~\citep{huang2020skipgnn_molecular_interactions}, social networks~\citep{monti2019fake_news}, recommendation systems~\citep{ying2018graph_recommended}, traffic prediction~\citep{ijcai2018_traffic_stgcn}, with graph neural networks (GNNs) becoming the standard for graph representations. However, in heterophilic setups, when nodes with different labels are likely to be connected, directly processing the graph structure leads to weak performance.
In ~\cite{bodnar_sheaf_diff}, they address this by attaching additional geometric structure to the graph, in the form of cellular sheaves~\cite{sheaf_curry}.

A cellular sheaf on graphs associates a vector space with each node and each edge
together with a linear projection between these spaces for each incident pair. To take into account this more complex geometric structure, SheafNN~\cite{sheafNN} generalised the classical GNNs~\cite{kipf2017semi_gcn,defferrard2016convolutional,DBLP:journals/corr/BrunaZSL13_bruna_lecun_spectral} by replacing the graph Laplacian with a sheaf Laplacian~\cite{hansen_sheaf_lap}. 
Higher-dimensional sheaf-based neural networks are explored, with sheaves either learned from the graph~\cite{bodnar_sheaf_diff} or deterministically inferred for efficiency~\cite{sheaf_conn_laplacian}. Recent methods integrate attention mechanisms~\cite{sheaf_att} or replace propagation with wave equations~\cite{sheaf_wave}. In recent developments, Sheaf Neural Networks have been found to significantly enhance the performance of recommendation systems, as they improve upon the limitations of graph neural networks \citep{purificato2023}.

In the domain of heterogeneous graphs, the concept of learning unique message functions for varying edges is well-established. However, there's a distinction in how sheaf-based methods approach this task compared to heterogeneous methods such as RGCN~\cite{schlichtkrull2017modeling}. Unlike the latter, which learns individual parameters for each kind of incident relationship, sheaf-based methods dynamically predict projections for each relationship, relying on features associated with the node and hyperedge. As a result, the total parameters in sheaf networks do not escalate with an increase in the number of hyperedges. This difference underscores a fundamental shift in paradigm between the two methods.

\textbf{Hypergraph Networks.} Graphs, while useful, have a strong limitation: they represent only pairwise relations. Many natural phenomena involve complex, higher-order interactions~\cite{brain_conectome,rxn_hypergraph,fish_schools,pollinators}, requiring a more general structure like hypergraphs.
Recent deep learning approaches have been developed for hypergraph structures. HyperGNN~\citep{HyperNN} expands the hypergraph into a weighted clique and applies message passing similar to GCNs~\cite{kipf2017semi_gcn}. HNHN~\citep{hnhn} improves this with non-linearities, while HyperGCN~\citep{HyperGCN} connects only the most discrepant nodes using a non-linear Laplacian. Similar to the trend in GNNs, attention models gain popularity also in the hypergraph domain. HCHA~\citep{HCHA} uses an attention-based incidence matrix, computed based on a nodes-hyperedge similarity. Similarly, HERALD~\citep{HERALD} uses a learnable distance to infer a soft incidence matrix. On the other hand, HEAT~\cite{HEAT} creates messages by propagating information inside each hyperedge using Transformers~\cite{transformer}.

Many hypergraph neural network (HNN) methods can be viewed as two-stage frameworks: 1) sending messages from nodes to hyperedges and 2) sending messages back from hyperedges to nodes. Thus, ~\citep{UniGNN} proposes a general framework where the first step is the average operator, while the second stage could use  any existing GNN module. Similarly, ~\citep{allset} uses either DeepSet functions~\cite{DeepSets} or Transformers~\cite{transformer} to implement the two stages, while~\citep{hyper_sage} uses a GNN-like aggregator in both stages, with distinct messages for each (node, hyperedge) pair. 

In contrast, we propose a novel model to improve the hypergraph processing by attaching a cellular sheaf to the hypergraph structure and diffusing the information inside the model according to it. We will first introduce the cellular sheaf for hypergraph, prove some properties for the associated Laplacians, and then propose and evaluate two architectures based on the sheaf hypergraph  Laplacians.

\begin{figure}[t!]
    \centering
    \includegraphics[scale=0.139]{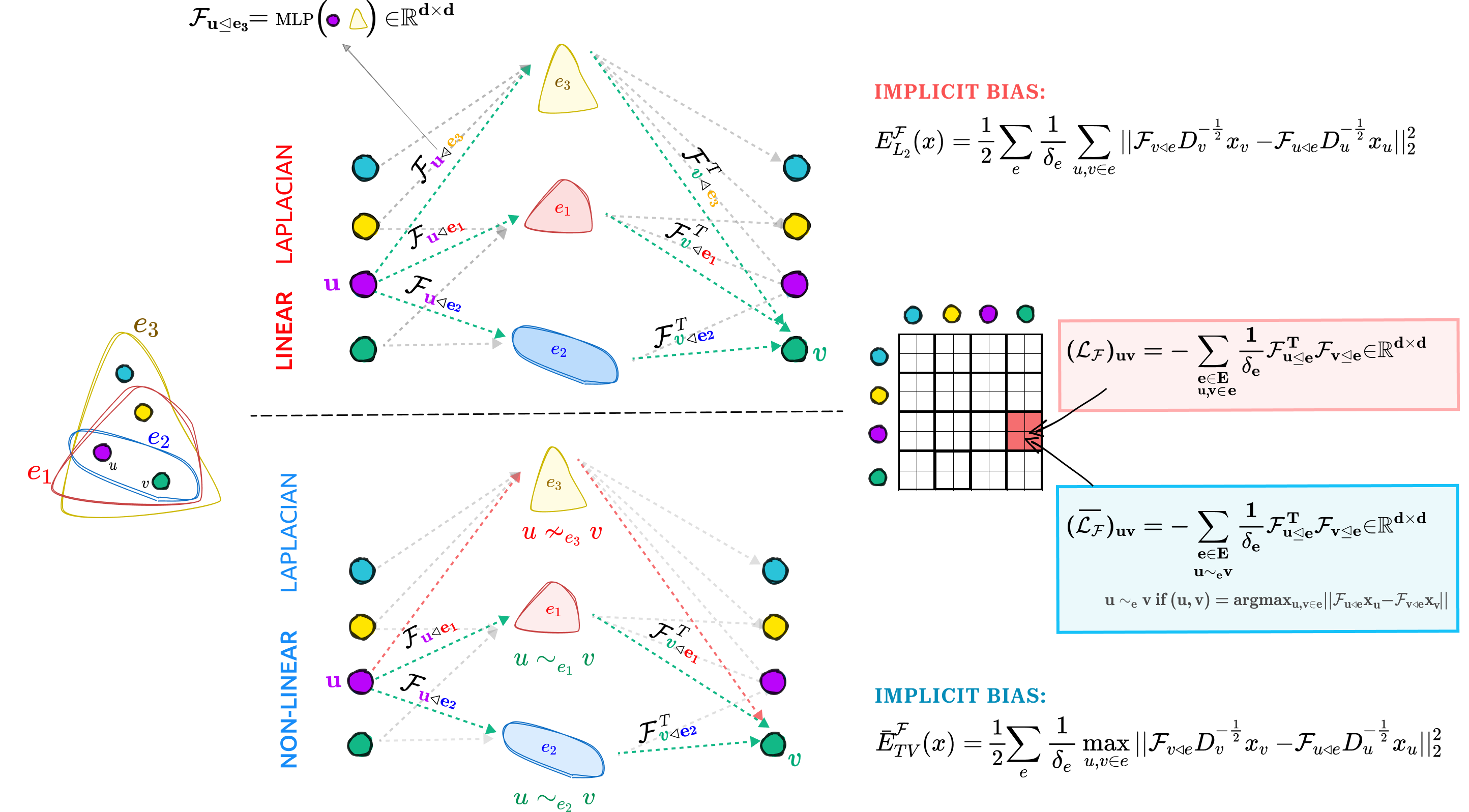}
    \caption{ Visual representation of linear and non-linear sheaf hypergraph Laplacian. (\textbf{Top}) In the linear case, the block matrix $(\mathcal{L_F})_{uv}$  corresponding to the pair of nodes $(u, v)$ accumulates contributions from each hyperedge that simultaneously contains both nodes. (\textbf{Bottom}) In the non-linear version, for each hyperedge, we first select the two nodes that are the most dissimilar in the hyperedge stalk domain: $u \sim_e v$ if $(u,v) = argmax_{u,v \in e} ||\mathcal{F}_{u \triangleleft e}x_u - \mathcal{F}_{v \triangleleft e}x_v||_2^2$. Then, the block matrix $(\mathcal{\bar L}_\mathcal{F})_{uv}$ associated with the pair of nodes $(u, v)$ only accumulates contributions from a hyperedge $e$ if $u \sim_e v$. The two operators (linear and non-linear sheaf hypergraph Laplacian) represent the building blocks for the Sheaf Hypergraph Neural Network and Sheaf Hypergraph Convolutional Network respectively and we theoretically show that they exhibit a more expressive implicit bias compared to the traditional Hypergraph Networks, leading to better performance.}
    \label{fig:fig_laplacian}
\end{figure}

\section{Hypergraph Sheaf Laplacian}

An undirected hypergraph is a tuple $\mathcal H = (V, E)$ where $V=\{1, 2 \dots n\}$ is a set of nodes (also called vertices), and $E$ is a set of hyperedges (also called edges when there is no confusion with the graph edges). Each hyperedge $e$ is a subset of the nodes set $V$. We denote by $n=|V|$ the number of nodes in the hypergraph $\mathcal H$ and by $m=|E|$ the number of hyperedges. In contrast to graph structures, where each edge contains exactly two nodes, in a hypergraph an edge $e$ can contain any number of nodes. The number of nodes in each hyperedge ($|e|$) is called the \textit{degree of the hyperedge} and is denoted by $\delta_e$. In contrast, the number of hyperedges containing each node $v$ is called the \textit{degree of the node} and is denoted by $d_v$. 


Following the same intuition from defining sheaves on graphs~\cite{sheafNN, bodnar_sheaf_diff}, we will introduce the cellular sheaf associated with a hypergraph $\mathcal{H}$.

\begin{definition} \label{def:sheaf}
A \textit{cellular sheaf $\mathcal{F}$ associated with a hypergraph $\mathcal{H}$} is defined as a triple $\langle \mathcal{F}(v),  \mathcal{F}(e), \mathcal{F}_{v \trianglelefteq e} \rangle$, where:

\begin{enumerate}[nosep]
    \item $\mathcal{F}(v)$ are \textit{vertex stalks:} vector spaces associated with each node $v$;
    \item $\mathcal{F}(e)$ are \textit{hyperedge stalks:} vector spaces associated with each hyperedge $e$; 
    \item $\mathcal{F}_{v \trianglelefteq e}: \mathcal{F}(v) \rightarrow \mathcal{F}(e)$ are \textit{restriction maps:} linear maps between each pair $v \trianglelefteq e$, if hyperedge $e$ contains node $v$.
\end{enumerate}
\end{definition}

In simpler terms, a sheaf associates a space with each node and each hyperedge in a hypergraph and also provides a linear projection that enables the movement of representations between nodes and hyperedges, as long as they are adjacent. Unless otherwise specified, we assign the same d-dimensional space for all vertex stalks $\mathcal{F}(v) = \mathbb{R}^d$ and all hyperedge stalks $\mathcal{F}(e) = \mathbb{R}^d$. We refer to $d$ as the dimension of the sheaf.

Previous works focused on creating hypergraph representations by relying on various methods of defining a Laplacian for a hypergraph. In this work, we will concentrate on two definitions: a linear version of the hypergraph Laplacian as used in~\cite{HyperNN}, and a non-linear version of the hypergraph Laplacian as in~\cite{HyperGCN}. We will extend both of these definitions to incorporate the hypergraph sheaf structure, analyze the advantages that arise from this, and propose two different neural network architectures based on each one of them. For a visual comparison between the two proposed sheaf hypergraph Laplacians, see Figure~\ref{fig:fig_laplacian}.

\subsection{Linear Sheaf Hypergraph Laplacian}




\begin{definition}
\label{def:sheaf_linear_lap}
Following the definition of a cellular sheaf on hypergraphs, we introduce the \textit{linear sheaf hypergraph Laplacian} associated with a hypergraph $\mathcal{H}$ as
$(\mathcal{L_F})_{vv} = \sum\limits_{e; v \in e} \frac{1}{\delta_e}\mathcal{F}^T_{v \trianglelefteq e}\mathcal{F}_{v \trianglelefteq e} \in \mathbb{R}^{d \times d}$ 
and 
$(\mathcal{L_F})_{uv} = -\sum\limits_{e; u,v \in e} \frac{1}{\delta_e} \mathcal{F}^T_{u \trianglelefteq e}  \mathcal{F}_{v \trianglelefteq e} \in \mathbb{R}^{d \times d}$, where $\mathcal{F}_{v \trianglelefteq e}: \mathbb{R}^d \rightarrow \mathbb{R}^d$ represents the linear restriction maps guiding the flow of information from node $v$ to hyperedge $e$.

\end{definition}
The linear sheaf Laplacian operator for node $v$ applied on a signal $x \in \mathbb{R}^{n \times d}$ can be rewritten as:
\begin{equation}
\label{linear_sheaf_laplacian}
\mathcal{L_F}(x)_v = \sum_{\substack{e;v \in e}} \frac{1}{\delta_e} \mathcal{F}^T_{v \trianglelefteq e}(\sum_{\substack{u \in e \\ u \neq v}} (\mathcal{F}_{v \trianglelefteq e}x_v - \mathcal{F}_{u \trianglelefteq e}x_u)).
\end{equation}

When each hyperedge contains exactly two nodes (thus $\mathcal{H}$ is a graph), the internal summation will contain a single term, and we recover the sheaf Laplacian for graphs as formulated in~\cite{bodnar_sheaf_diff}.

On the other hand, for the trivial sheaf, when the vertex and hyperedge stalks are both fixed to be $\mathbb{R}$ and the restriction map is the identity $\mathcal{F}_{v \trianglelefteq e} = 1$ we recover the usual linear hypergraph Laplacian ~\cite{HyperNN,hyper_att} defined as $\mathcal{L}(x)_v = \sum_{e;v \in e} \frac{1}{\delta_e} \sum_{u \in e} (x_v - x_u)$.  
However, when we allow for higher-dimensional stalks $\mathbb{R}^d$, the restriction maps for each adjacency pair $(v, e)$ become linear projections $\mathcal{F}_{v \trianglelefteq e} \in \mathbb{R}^{d \times d}$, enabling us to model more complex propagations, customized for each incident (node, hyperedge) pairs.

In the following sections, we will demonstrate the advantages of using this sheaf hypergraph diffusion instead of the usual hypergraph diffusion.

\textbf{Reducing energy via linear diffusion.}
Previous work~\cite{oversmoothing_hyper} demonstrates that diffusion using the classical symmetric normalised version of the hypergraph Laplacian $\Delta = D^{-\frac{1}{2}} \mathcal{L} D^{-\frac{1}{2}}$, where $D$ is a diagonal matrix containing the degrees of the vertices, reduces the following energy function: $E_{L_2}(x) =  \frac{1}{2}\sum_e \frac{1}{\delta_e}\sum_{u,v \in e}||d_v^{-\frac{1}{2}}x_v - d_u^{-\frac{1}{2}}x_u||_2^2$. Intuitively, this means that applying diffusion using the \textit{linear hypergraph Laplacian} leads to similar representations for neighbouring nodes. While this is desirable in some scenarios, it may cause poor performance in others, a phenomenon known as over-smoothing~\cite{oversmoothing_hyper}. In the following, we show that applying diffusion using the linear \textit{sheaf} hypergraph Laplacian addresses these limitations by implicitly minimizing a more expressive energy function. This allows us to model phenomena that were not accessible using the usual Laplacian.

\begin{definition}[]
We define \textit{sheaf Dirichlet energy} of a signal $x \in \mathbb{R}^{n \times d}$ \textit{on a hypergraph} $\mathcal{H}$ as:
\begin{equation*}
    E^{\mathcal{F}}_{L_2}(x) =  \frac{1}{2}\sum_e \frac{1}{\delta_e}\sum_{u,v \in e}||\highlight{\mathcal{F}_{v \tri e}}D_v^{-\frac{1}{2}}x_v -\highlight{\mathcal{F}_{u \tri e}}D_u^{-\frac{1}{2}}x_u||_2^2,
\end{equation*}
where $D_v = \sum\limits_{e; v \in e} \mathcal{F}^T_{v \trianglelefteq e}\mathcal{F}_{v \trianglelefteq e}$ is a normalisation term equivalent to the nodes degree $d_v$ for the trivial sheaf and $D = diag(D_1, D_2 \dots D_n)$ the corresponding block diagonal matrix.
\end{definition}

This quantity measures the discrepancy between neighbouring nodes in the hyperedge stalk domain as opposed to the usual Dirichlet energy for hypergraphs that, instead, measures this distance in the node features domain. In the following, we are showing that, applying hypergraph diffusion using the linear sheaf Laplacian implicitly reduces this energy.


\begin{proposition}[]
\label{prop:linear_min_energ}
 The diffusion process using a symmetric normalised version of the linear sheaf hypergraph Laplacian minimizes the sheaf Dirichlet energy of a signal $x$ on a hypergraph $\mathcal{H}$. Moreover, the energy decreases with each layer of diffusion. 
\end{proposition}
 Concretely, defining the diffusion process as $Y = (I - \Delta^{\mathcal{F}})X$ where $\Delta^{\mathcal{F}} = D^{-\frac{1}{2}}\mathcal{L^F}D^{-\frac{1}{2}}  \in \mathbb{R}^{nd \times nd}$ represents the symmetric normalised version of the linear sheaf hypergraph Laplacian, we have that $E^{\mathcal{F}}_{L_2}(Y) < \lambda_{*}E^{\mathcal{F}}_{L_2}(X)$, with $\{\lambda_i\}$ the non-zero eigenvalues of $\Delta^{\mathcal{F}}$ and $\lambda_* = \max_{i}{\{(1-\lambda_i)^2\}} < 1$. All the proofs are in the Supplementary Material.

This result addresses some of the limitations of standard hypergraph processing. First, while classical diffusion using hypergraph Laplacian brings closer representations of the nodes in the nodes space ($x_v$, $x_u$), our linear sheaf hypergraph Laplacian allows us to bring closer representations of the nodes in the more complex space associated with the hyperedges ($\mathcal{F}_{v \tri e}x_v$, $\mathcal{F}_{u \tri e}x_u$). This encourages a form of hyperedge agreement, while preventing the nodes to become uniform. Secondly, in the hyperedge stalks, each node can have a different representation for each hyperedge it is part of, leading to a more expressive processing compared to the classical methods. Moreover, in many Hypergraph Networks, the hyperedges uniformly aggregate information from all its components. Through the presence of a restriction map for each (node, hyperedge) pair, we enable the model to learn the individual contribution that each node sends to  each hyperedge.  

From an opinion dynamic perspective~\cite{opinion_dyn} when the hyperedges represent group discussions, the input space $x_v$ can be seen as the private opinion, while the hyperedge stalk $\mathcal{F}_{v \tri e}x_v$ can be seen as a public opinion (what an individual $v$ decide to express in a certain group $e$). Minimizing the \textit{Dirichlet energy} creates private opinions that are in consensus inside each hyperedge, while minimizing the \textit{sheaf Dirichlet energy} creates an \textit{apparent} consensus, by only uniformizing the expressed opinion. Through our sheaf setup, each individual is allowed to express varying opinion in each group it is part of, potentially different than their personal belief. This introduces a realistic scenario inaccessible in the original hypergraph diffusion setup.



\subsection{Non-Linear Sheaf Hypergraph Laplacian}

Although the linear hypergraph Laplacian is commonly used to process hypergraphs, it falls short in fully preserving the hypergraph structure~\cite{hol_graphs}.
To address these shortcomings, ~\cite{total_var_hyper} introduces the non-linear Laplacian, demonstrating that its spectral properties are more suited for higher-order processing compared to the linear Laplacian. For instance,  compared to the linear version, the non-linear Laplacian leads to a more balanced partition in the minimum cut problem, a task known to be tightly related to the semi-supervised node classification. 
Additionally, while the linear Laplacian associates a clique for each hyperedge, the non-linear one offers the advantage of relying on a much sparser connectivity.
We will adopt a similar methodology to derive the non-linear version of the sheaf hypergraph Laplacian and analyze the benefits of applying diffusion using this operator.





\begin{definition}\label{def:sheaf_nonlinear_lap} We introduce the \textit{non-linear sheaf hypergraph Laplacian} of a hypergraph $\mathcal{H}$ with respect to a signal $x$ as following:

\begin{enumerate}
    \item For each hyperedge $e$, compute  $(u_e, v_e) = argmax_{u,v \in e} ||\mathcal{F}_{u \triangleleft e}x_u - \mathcal{F}_{v \triangleleft e}x_v||$, the set of pairs containing the nodes with the most discrepant features in the hyperedge stalk. 
    \item Build an undirected graph $\mathcal{G}_H$ containing the same sets of nodes as $\mathcal{H}$  and, for each hyperedge $e$ connects the most discrepant nodes $(u, v)$ (from now on we will write $u \sim_e v$ if they are connected in the $\mathcal{G}_H$ graph due to the hyperedge e). If multiple pairs have the same maximum discrepancy, we will randomly choose one of them. 
    \item Define the sheaf non-linear hypergraph Laplacian as:
        \begin{equation} \label{eq_def_non_linear_sheaf_lap}
        \mathcal{\bar L_F}(x)_v = \sum_{e; u \sim_e v}\frac{1}{\delta_e}\mathcal{F}^T_{v \trianglelefteq e}\big(\mathcal{F}_{v \trianglelefteq e}x_v-\mathcal{F}_{u \trianglelefteq e}x_u \big).
        \end{equation}
\end{enumerate}
\end{definition}
Note that the sheaf structure impacts the non-linear diffusion in two ways: by shaping the graph structure creation (Step 1), where the two nodes with the greatest distance in the hypergraph stalk are selected rather than those in the input space; and by influencing the information propagation process (Step 3). When the sheaf is restricted to the trivial case ($d=1$ and $\mathcal{F}_{v \trianglelefteq e}=1$) this corresponds to the non-linear sheaf Laplacian of a hypergraph as introduced in~\cite{total_var_hyper}.

\paragraph{Reducing total variation via non-linear diffusion.} 

In the following discussion, we will demonstrate how transitioning from a linear to a non-linear sheaf hypergraph Laplacian alters the energy guiding the inductive bias. This phenomenon was previously investigated for the classical hypergraph Laplacian, with~\cite{total_var_hyper} revealing enhanced expressivity in the non-linear case.

\begin{definition} We define  \textit{the sheaf total variation} of a signal $x \in \mathbb{R}^{n \times d}$ \textit{on a hypergraph} $\mathcal{H}$ as:

\begin{equation*}
\bar E^{\mathcal{F}}_{TV}(x) =  \frac{1}{2}\sum_e\frac{1}{\delta_e}\max_{u,v \in e}||\highlight{\mathcal{F}_{v \tri e}}D_v^{-\frac{1}{2}}x_v -\highlight{\mathcal{F}_{u \tri e}}D_u^{-\frac{1}{2}}x_u||_2^2,
\end{equation*}
where $D_v = \sum\limits_{e; v \in e} \mathcal{F}_{v \trianglelefteq e}^T\mathcal{F}_{v \trianglelefteq e}$ is a normalisation term equivalent to the node's degree in the classical setup and $D = diag(D_1, D_2 \dots D_n)$ is the corresponding block diagonal matrix.
\end{definition}

 This quantity generalised the total variance (TV) $ \bar E_{TV}(x) =  \frac{1}{2}\sum_e \frac{1}{\delta_e}\max_{u,v \in e}||d_v^{-\frac{1}{2}}x_v -d_u^{-\frac{1}{2}}x_u||_2^2$ minimized in the non-linear hypergraph label propagation~\cite{total_var_hyper,pmlr-v70-zhang17d}. Unlike the TV,  the sheaf total variation measures the highest discrepancy at the hyperedge level computed in the hyperedge stalk, as opposed to the TV, which gauges the highest discrepancy in the feature space. We will explore the connection between the sheaf TV and our \textit{non-linear sheaf hypergraph diffusion}.

\begin{proposition}[]
 The diffusion process using the symmetric normalised version of non-linear sheaf hypergraph Laplacian minimizes the sheaf total variance of a signal $x$ on hypergraph $\mathcal{H}$. 
\end{proposition}


Despite the change in the potential function being minimized, the overarching objective remains akin to that of the linear case: striving to achieve a coherent consensus among the representations within the hyperedge stalk space, rather than generating uniform features for each hyperedge in the input space. In contrast to the linear scenario, where a quadratic number of edges is required for each hyperedge, the non-linear sheaf hypergraph Laplacian associates a single edge with each hyperedge, thereby enhancing computational efficiency.

\subsection{Sheaf Hypergraph Networks}
Popular hypergraph neural networks~\cite{hyper_att, HyperNN, HyperGCN,wang2022equivariant} draw inspiration from a variety of hypergraph diffusion operators~\cite{hol_graphs,total_var_hyper,Saito_2018}, giving rise to diverse message passing techniques. These techniques all involve the propagation of information from nodes to hyperedges and vice-versa. We will adopt a similar strategy and introduce the Sheaf Hypergraph Neural Network and Sheaf Hypergraph Convolutional Network, based on two message-passing schemes inspired by the sheaf diffusion mechanisms discussed in this paper.

Given a hypergraph $\mathcal{H}=(V,E)$ with nodes characterised by a set of features $ X \in \mathbb{R}^{n \times f}$, we initially linearly project the input features into $\tilde{X} \in \mathbb{R}^{n \times (df)}$ and then reshape them into $ \tilde{X} \in \mathbb{R}^{nd \times f}$. As a result, each node is represented in the vertex stalk as a matrix $\mathbb{R}^{d \times f}$, where $d$ denotes the dimension of the vertex stalk, and $f$ indicates the number of channels. 

A general layer of Sheaf Hypergraph Network is defined as: 
\begin{equation*}
Y = \sigma((I_{nd}-\overset{\bullet}\Delta)(I_n \otimes W_1 )\tilde XW_2).
\end{equation*}
Here, $\overset{\bullet}\Delta$ can be either $\Delta^{\mathcal{F}}=D^{-\frac{1}{2}} \mathcal{L}^{\mathcal{F}} D^{-\frac{1}{2}} $ for the \textit{linear} sheaf hypergraph Laplacian introduced in Eq.~\ref{linear_sheaf_laplacian} or $\bar\Delta^{\mathcal{F}}=D^{-\frac{1}{2}} \bar{\mathcal{L}}^{\mathcal{F}} D^{-\frac{1}{2}}$ for the \textit{non-linear} sheaf hypergraph Laplacian introduced in Eq.~\ref{eq_def_non_linear_sheaf_lap}. Both $W_1 \in \mathbb{R}^{d \times d}$ and $W_2 \in \mathbb{R}^{f \times f}$ are learnable parameters, while $\sigma$ represents ReLU non-linearity.\\

\textbf{Sheaf Hypergraph Neural Network} (SheafHyperGNN). This model utilizes the \textit{linear} sheaf hypergraph Laplacian $\overset{\bullet}\Delta = \Delta^{\mathcal{F}}$. When the sheaf is trivial ($d=1$ and $\mathcal{F}_{v \trianglelefteq e}=1$), and $W_1=\mathbf{I}_d$, the SheafHyperGNN is equivalent to the conventional HyperGNN architecture~\cite{HyperNN}. However, by increasing dimension $d$ and adopting dynamic restriction maps, our proposed SheafHyperGNN becomes more expressive. For every adjacent node-hyperedge pair $(v,e)$, we use a $d \times d$ block matrix to discern each node's contribution instead of a fixed weight that only stores the incidence relationship. The remaining operations are similar to those in HyperGNN~\cite{HyperNN}. More details on how the block matrices $\mathcal{F}_{v \trianglelefteq e}$ are learned can be found in the following subsection.

\textbf{Sheaf Hypergraph Convolutional Network} (SheafHyperGCN).
This model employs the non-linear Laplacian $\overset{\bullet}\Delta = \bar\Delta^{\mathcal{F}}$. Analogous to the linear case, when the sheaf is trivial and $W_1=\mathbf{I}_d$ we obtain the classical HyperGCN architecture~\cite{HyperGCN}. In our experiments, we will use an approach similar to that in~\cite{HyperGCN} and adjust the Laplacian to include mediators. This implies that we will not only connect the two most discrepant nodes but also create connections between each node in the hyperedge and these two most discrepant nodes, resulting in a denser associated graph. For more information on this variation, please refer to~\cite{HyperGCN} or Supplementary Material.

In summary, the models introduced in this work, SheafHyperGNN and SheafHyperGCN serve as generalisations of the classical HyperGNN~\cite{HyperNN} and HyperGCN~\cite{HyperGCN}. These new models feature a more expressive implicit regularisation compared to their traditional counterparts. 

\begin{table}[t]
  \caption{\textbf{Performance on a collection of hypergraph benchmarks.} 
  Our models using sheaf hypergraph Laplacians demonstrate a clear advantage over their counterparts using classical Laplacians (HyperGNN and HyperGCN). Compared to other recent methods, SheafHyperGNN and SheafHyperGCN achieve competitive performance and attain state-of-the-art results in five of the datasets.
  }
  \label{tab:sota}
  \centering
  \resizebox{\columnwidth}{!}{%
  \begin{tabular}{ccccccccc}
    \toprule
     Name      & Cora & Citeseer & Pubmed & Cora\_{CA} & DBLP\_{CA} & Senate & House & Congress \\
    \midrule
     HCHA & 79.14 \pm 1.02 & 72.42 \pm 1.42 & 86.41 \pm 0.36 & 82.55 \pm 0.97 & 90.92 \pm 0.22  & 48.62 \pm 4.41 & 61.36 \pm 2.53 &  90.43 \pm 1.20 \\
     HNHN & 76.36 \pm 1.92 & 72.64 \pm 1.57  & 86.90 \pm 0.30 &  77.19 \pm 1.49 & 86.78 \pm 0.29 & 50.93 \pm 6.33 & 67.8 \pm 2.59 & 53.35 \pm 1.45 \\
     AllDeepSets & 76.88 \pm 1.80 & 70.83 \pm 1.63 & 88.75 \pm 0.33 & 81.97 \pm 1.50 & 91.27 \pm 0.27 & 48.17 \pm 5.67 & 67.82 \pm 2.40 & 91.80 \pm 1.53 \\
     AllSetTransformers & 78.58 \pm 1.47 & 73.08 \pm 1.20 & 88.72 \pm 0.37 & 83.63 \pm 1.47 & 91.53 \pm 0.23 & 51.83 \pm 5.22 & 69.33 \pm 2.20 & 92.16 \pm 1.05 \\
     UniGCNII & 78.81 \pm 1.05 & 73.05 \pm 2.21 & 88.25 \pm 0.33  & 83.60 \pm 1.14 &  91.69 \pm 0.19 & 49.30 \pm 4.25 & 67.25 \pm 2.57 & 94.81 \pm 0.81 \\
     HyperND & 79.20 \pm 1.14 &  72.62 \pm 1.49 & 86.68 \pm 0.43 & 80.62 \pm 1.32 & 90.35 \pm 0.26 & 52.82 \pm 3.20 & 51.70 \pm 3.37 & 74.63 \pm 3.62 \\
     
     ED-HNN & 80.31 \pm 1.35 & 73.70 \pm 1.38 & \textbf{\color{black}{89.03 \pm 0.53}} & 83.97 \pm 1.55 & \textbf{\color{black}{91.90 \pm 0.19}} & 64.79 \pm 5.14 & 72.45 \pm 2.28 & \textbf{\color{black}{95.00 \pm 0.99}} \\
     \midrule
     HyperGCN\footnotemark & 78.36 \pm 2.01 & 71.01 \pm 2.21 & 80.81 \pm 12.4 & 79.50 \pm 2.11 & 89.42 \pm 0.16* & 51.13 \pm 4.15 & 69.29 \pm 2.05 & 89.67 \pm 1.22 \\
     SheafHyperGCN & 80.06 \pm 1.12 & 73.27 \pm 0.50 & 87.09 \pm 0.71 & 83.26 \pm 1.20 & 90.83 \pm 0.23 & 66.33 \pm 4.58 & 72.66 \pm 2.26 & 90.37 \pm 1.52 \\
     \midrule
     HyperGNN & 79.39 \pm 1.36 & 72.45 \pm 1.16  & 86.44 \pm 0.44 & 82.64 \pm 1.65 & 91.03 \pm 0.20 & 48.59 \pm 4.52 & 61.39 \pm 2.96 & 91.26 \pm 1.15 \\
     SheafHyperGNN & \textbf{\color{black}{81.30} \pm 1.70} &  \textbf{\color{black}{74.71 \pm 1.23}} & 87.68 \pm 0.60 & \textbf{\color{black}{85.52 \pm 1.28}}  & 91.59 \pm 0.24 & \textbf{\color{black}{68.73 \pm 4.68}} & \textbf{\color{black}{73.84 \pm 2.30}} & 91.81 \pm 1.60 \\
    \bottomrule
  \end{tabular}%
  }
\end{table}
\footnotetext[1]{Results where rerun compared to~\cite{wang2022equivariant} using the same hyperparameters, to fix an existing issue in the original code.}

\paragraph{Learnable Sheaf Laplacian.}\label{learn_reduction}
A key advantage of Sheaf Hypergraph Networks lies in attaching and processing a more complex structure (sheaf) instead of the original standard hypergraph. Different sheaf structures can be associated with a single hypergraph, and accurately modeling the most suitable structure is crucial for obtaining effective and meaningful representation. In our proposed models, we achieve this by designing learnable restriction maps. 
For a d-dimensional sheaf, we predict the restriction maps for each pair of incident (vertex v, hyperedge e) as $ \mathcal{F}_{v \trianglelefteq e} = \text{MLP}(x_v || h_e) \in \mathbb{R}^{d^2}$, where $x_v$ represent node features of $v$, and $h_e$ represents features of the hyperedge $e$. This vector representation is then reshaped into a $d \times d$ block matrix representing the linear restriction map for the $(v,e)$ pair. 
When hyperedge features $h_e$ are not provided, any permutation-invariant operation can be applied to obtain hyperedge features from node-level features. We experiment with three types of $d \times d$ block matrices: diagonal, low-rank and general matrices, with the diagonal version consistently outperforming the other two. These restriction maps are further used to define the sheaf hypergraph Laplacians (Def.~\ref{def:sheaf_linear_lap}, or ~\ref{def:sheaf_nonlinear_lap}) used in the final Sheaf Hypergraph Networks. 
Please refer to the Supplementary Material for more details on how we constrain the restriction maps.

\section{Experimental Analysis}

We evaluate our model on eight real-world datasets that vary in domain, scale, and heterophily level and are commonly used for benchmarking hypergraphs. These include Cora, Citeseer, Pubmed, Cora-CA, DBLP-CA~\cite{HyperGCN},  House~\cite{house_data}, Senate and Congress~\cite{senate_congress_data}. To ensure a fair comparison with the baselines, we follow the same training procedures used in~\cite{wang2022equivariant} by randomly splitting the data into $50\%$ training samples, $25\%$ validation samples and $25\%$ test samples, and running each model $10$ times with different random splits. We report average accuracy along with the standard deviation. 

Additionally, we conduct experiments on a set of synthetic heterophilic datasets inspired by those introduced by~\cite{wang2022equivariant}. Following their approach, we generate a hypergraph using the contextual hypergraph stochastic block model~\cite{contextual_block_model,contextual_block_model2,contextual_block_model3}, containing $5000$ nodes: half belong to class $0$ while the other half to class $1$. We then randomly sample $1000$ hyperedges with a cardinality $15$, each containing exactly $\beta$ nodes from class $0$. The heterophily level is computed as $\alpha = \text{min}(\beta, 15-\beta$). Node features are sampled from a label-dependent Gaussian distribution with a standard deviation of $1$. As the original dataset is not publicly available, we generate our own set of datasets by varying the heterophily level $\alpha \in \{ 1 \dots 7 \}$ and rerun their experiments for a fair comparison. 

The experiments are executed on a single NVIDIA Quadro RTX 8000 with 48GB of GPU memory. Unless otherwise specified, our results represent the best performance obtained by each architecture using hyper-parameter optimisation with random search. Details on all the model choices and hyper-parameters can be found in the Supplementary Material. 


    

\textbf{Laplacian vs Sheaf Laplacian.} As demonstrated in the previous section, SheafHyperGNN and SheafHyperGCN are generalisations of the standard HyperGNN~\cite{HyperNN} and HyperGCN~\cite{HyperGCN}, respectively. They transition from the trivial sheaf ($d=1$ and $\mathcal{F}_{v \triangleleft e}=1$) to more complex structures ($d \geq1$ and $\mathcal{F}_{v \triangleleft e}$ a $d \times d$ learnable projection). The results in Table~\ref{tab:sota} and Table~\ref{tab:hetero_dataset} show that both models significantly outperform their counterparts on all tested datasets. Among our models, the one based on linear Laplacian (SheafHyperGNN) consistently outperforms the model based on non-linear Laplacian (SheafHyperGCN) across all datasets. This observation aligns with the performance of the models based on standard hypergraph Laplacian, where HyperGCN is outperformed by HyperGNN in all but two real-world datasets, despite their theoretical advantage~\cite{total_var_hyper}.

\textbf{Comparison to recent methods.} We also compare to several recent models from the literature such as HCHA~\cite{HCHA}, HNHN~\cite{hnhn}, AllDeepSets~\cite{allset}, AllSetTransformer~\cite{allset}, UniGCNII~\cite{ijcai21-UniGNN}, HyperND~\cite{hyperND}, and ED-HNN~\cite{wang2022equivariant}. Our models achieve competitive results on all real-world datasets, with state-of-the art performance on five of them (Table~\ref{tab:sota}). These results confirm the advantages of using the sheaf Laplacians for processing hypergraphs. We also compare our models against a series baselines on the synthetic heterophilic dataset. The results are shown in Table~\ref{tab:hetero_dataset}. Our best model, SheafHyperGNN, consistently outperforms the other models across all levels of heterophily. Note that, our framework enhancing classical hypergraph processing with sheaf structure is not restricted to the two traditional models tested in this paper (HyperGNN and HyperGCN). Most of the recent state-of-the-art methods, such as ED-HNN, could be easily adapted to learn and process our novel cellular sheaf hypergraph instead of the standard hypergraph, leading to further advancement in the hypergraph field. 

\begin{table}[t]
  \caption{{\textbf{Ablation study on Restriction Maps}: we explore three types of $d \times d$ restriction maps: diagonal, low-rank and general. Diagonal matrices consistently achieve better accuracy on most of the datasets, demonstrating a superior balance between complexity and expressivity
  }
  }
  \label{tab:ablation_red_map}
  \centering
  \resizebox{\columnwidth}{!}{%
  \begin{tabular}{ccccccccc}
    \toprule
     Name      & Cora & Citeseer & Pubmed & Cora\_{CA} & DBLP\_{CA} & Senate & House & Congress \\
     \toprule
        Diag-SheafHyperGCN & 80.06 \pm 1.12 & 73.27 \pm 0.50 & 87.09 \pm 0.71 & 83.26 \pm 1.20 & 90.83 \pm 0.23 & 66.33 \pm 4.58 & 72.66 \pm 2.26 & 90.37 \pm 1.52 \\
     LR-SheafHyperGCN & 78.70 \pm 1.14 & 72.14 \pm 1.09 & 86.99 \pm 0.39  & 82.61 \pm 1.28 & 90.84 \pm 0.29 & 66.76 \pm 4.58 & 70.70 \pm 2.23 & 84.88 \pm 2.31\\
     Gen-SheafHyperGCN & 79.13 \pm 0.85 & 72.54 \pm 2.3 & 86.90 \pm 0.46 & 82.54 \pm 2.08 & 90.57 \pm 0.40 & 65.49 \pm 5.17 & 71.05 \pm 2.12 & 82.14 \pm 2.81 \\
    \midrule
    Diag-SheafHyperGNN & 81.30 \pm 1.70 &  74.71 \pm 1.23 & 87.68 \pm 0.60 & 85.52 \pm 1.28  & 91.59 \pm 0.24 & 68.73 \pm 4.68 & 73.62 \pm 2.29 & 91.81 \pm 1.60 \\
     LR-SheafHyperGNN & 76.65 \pm 1.41 & 74.05 \pm 1.34 &  87.09 \pm 0.25 & 77.05 \pm 1.00 & 85.13 \pm 0.29 & 68.45 \pm 2.46 & 73.84 \pm 2.30 & 74.83 \pm 2.32 \\
     Gen-SheafHyperGNN & 76.82 \pm 1.32 & 74.24 \pm 1.05 & 87.35 \pm 0.34 & 77.12 \pm 1.14 & 84.99 \pm 0.39 & 68.45 \pm 4.98 & 69.47 \pm 1.97 & 74.52 \pm 1.27\\
    \bottomrule
\vspace{-8mm}
  \end{tabular}%
  }
\end{table}
\begin{figure}[t]
\begin{minipage}{0.47\textwidth}
\centering
    \small
    \thickmuskip=0mu
    \caption{\textbf{Impact of Depth and Stalk Dimension} 
    evaluated on the heterophilic dataset ($\alpha~=~7$). SheafHyperGNN's performance is unaffected by increasing depth, and high-dimensional stalks is essential for achieving top performance. The Dirichlet energy shows that, while HyperGNN enforces the nodes to be similar, our SheafHyperGNN does not suffer from this limitation, encouranging features diversity. 
    }
    \vspace{-2mm}
     \includegraphics[width=1.0\textwidth]{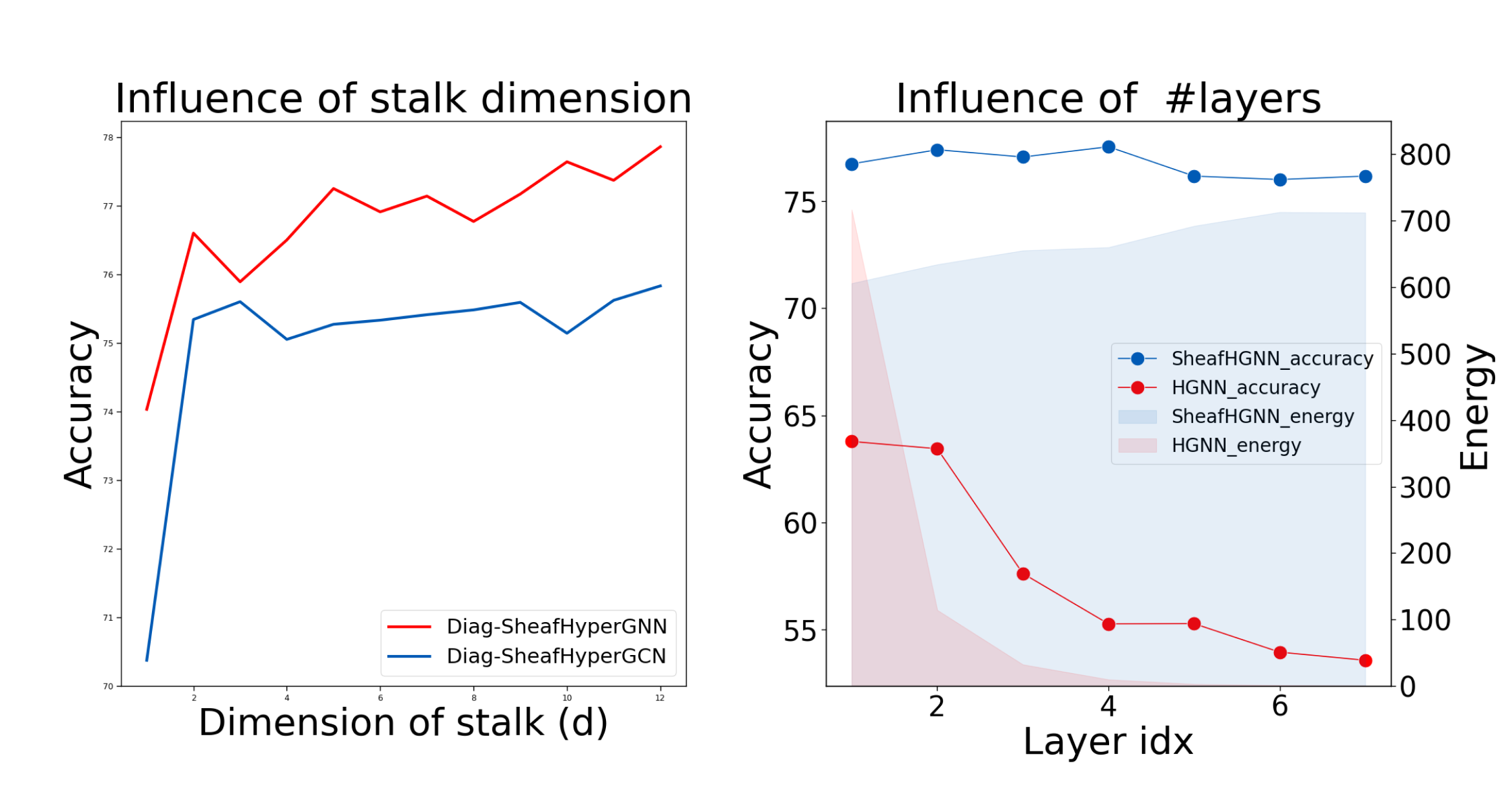}
     \label{fig:ablations_stack_dim}
\end{minipage}
\hfill
\begin{minipage}{0.5\textwidth}
\centering
    \tiny
    \vspace{-11mm}
    \captionof{table}{\textbf{Accuracy on Synthetic Datasets with Varying Heterophily Levels}
    : Across all different level of heterophily ($\alpha$), our sheaf-based methods SheafGCN and SheafHGNN consistently outperform their counterparts. Additionally, they achieve top results for all heterophily levels, further demonstrating their effectiveness. For each experiment, the result represents average accuracy over $10$ runs.\\
    }
  \label{tab:hetero_dataset}
  \begin{tabular}{cccccccc}
    \toprule
    & \multicolumn{7}{c}{heterophily ($\alpha$)}\\
     Name      & $1$ &  $2$ & $3$ &  $4$ &  $5$ &  $6$ & $7$  \\
    \midrule
    HyperGCN  & 83.9  &	69.4 & 72.9 & 75.9 & 70.5 & 67.3 & 66.5 \\ 
    HyperGNN  &  98.4 &  83.7 & 79.4 & 74.5 & 69.5 & 66.9 & 63.8 \\ 
    HCHA  &  98.1 &  81.8 & 78.3 & 75.88 & 74.1 & 71.1 & 70.8 \\ 
    ED-HNN & 99.9 & 91.3 & 88.4 & 84.1 & 80.7 & 78.8 & 76.5 \\ 
    \midrule
    SheafHGCN  &  \textbf{100} &	87.1 & 84.8 & 79.2 & 78.1 & 76.6 & 75.5 \\ 
    SheafHGNN  &  \textbf{100} &	\textbf{94.2} & \textbf{90.8} & \textbf{86.5} & \textbf{82.1} & \textbf{79.8} & \textbf{77.3} \\ 
    \bottomrule  
    
  \end{tabular}%
\end{minipage}
\vspace{-10mm}
\end{figure}

In the following sections, we conduct a series of ablation studies to gain a deeper understanding of our models. We will explore various types of restriction maps, analyze how performance changes when varying the network depth and study the importance of stalk dimension for final accuracy.

\textbf{Investigating the Restriction Maps.} Both linear and non-linear sheaf hypergraph Laplacians rely on attaching a sheaf structure to the hypergraph. For a cellular sheaf $\mathcal{F}$ with vertex stalks $\mathcal{F}(v) = \mathbb{R}^d$ and hyperedge stalks $\mathcal{F}(e) = \mathbb{R}^d$ as used in our experiments, this involves inferring the restriction maps $\mathcal{F}_{v \triangleleft e} \in \mathbb{R}^{d \times d}$ for each incidence pair $(v, e)$. We implement these as a function dependent on corresponding nodes and hyperedge features: $ \mathcal{F}_{v \trianglelefteq e} = \text{MLP}(x_v || h_e) \in \mathbb{R}^{d^2}$. Learning these matrices can be challenging; therefore, we experimented with adding constraints to the type of matrices used as restriction maps. In Table~\ref{tab:ablation_red_map} we show the performance obtained by our models when constraining the restriction maps to be either diagonal (Diag-SheafHyperNN), low-rank (LR-SheafHyperNN) or general matrices (Gen-SheafHyperNN). We observe that the sheaves equipped with diagonal restriction maps perform better than the more general variations. We believe that the advantage of the diagonal restriction maps is due to easier optimization, which overcomes the lose in expressivity. More details about predicting constrained $d \times d$ matrices can be found in the Supplementary Material.

\textbf{Importance of Stalk Dimension.} The standard hypergraph Laplacian corresponds to a sheaf Laplacian with $d=1$ and $\mathcal{F}_{v \trianglelefteq e}=1$. Constraining the stalk dimension to be $1$, but allowing the restriction maps to be dynamically predicted, becomes similar to an attention mechanism~\cite{HCHA}. However, attention models are restricted to guiding information via a scalar probability, thus facing the same over-smoothing limitations as traditional HyperGNN in the heterophilic setup. Our d-dimensional restriction maps increase the model's expressivity by enabling more complex information transfer between nodes and hyperedges, tailored for each individual pair. We validate this experimentally on the synthetic heterophilic dataset, using the diagonal version of the models, which achieves the best performance in the previous ablation. In Figure~\ref{fig:ablations_stack_dim}, we demonstrate how performance significantly improves when allowing higher-dimensional stalks ($d > 1$). These results are consistent for both linear sheaf Laplacian-based models (SheafHyperGNN) and non-linear ones (SheafHyperGCN).

\textbf{Influence of Depth.} It is well-known that stacking many layers in a hypergraph network can lead to a decrease in model performance, especially in the heterophilic setup. This phenomenon, called over-smoothing, is well-studied in both graph~\cite{oversmoothing_graph} and hypergraph literature~\cite{oversmoothing_hyper}. To analyse the extent to which our model suffers from this limitation, we train a series of models on the most heterophilic version of the synthetic dataset ($\alpha=7$). For both SheafHyperGNN and its HyperGNN equivalent, we vary the number of layers between $1-8$. In Figure~\ref{fig:ablations_stack_dim}, we observe that while HyperGNN exhibits a drop in performance when going beyond $3$ layers, SheafHyperGNN's performance remains mostly constant. Similar results were observed for the non-linear version when comparing SheafHyperGCN with HyperGCN (results in Supplementary Material). These results indicates potential advantages of our models in the heterophilic setup by allowing the construction of deeper architectures.

\textbf{Investigating Features Diversity.} Our theoretical analysis shows that, while conventional Hypergraph Networks tend to produce similar features for neighbouring nodes, our Sheaf Hypergraph Networks reduce the distance between neighbouring nodes in the more complex hyperedge stalk space. As a result, the nodes' features do not become uniform, preserving their individual identities.  We empirically evaluate this,  by computing the Dirichlet energy for HyperGNN and SheafHyperGNN (shaded area in Figure~\ref{fig:ablations_stack_dim}), as a measure of similarity between neighbouring nodes. The results are aligned with the theoretical analysis: while increasing depth in HyperGNN creates uniform features, SheafHyperGNN does not suffer from this limitation, encouraging diversity between the nodes. 

\section{Conclusion}
In this paper we introduce the cellular sheaf for hypergraphs, an expressive tool for modelling higher-order relations build upon the classical hypergraph structure. Furthermore, we propose two models capable of inferring and processing the sheaf hypergraph structure, based on linear and non-linear sheaf hypergraph Laplacian, respectively. We prove that the diffusion processes associated with these models induce a more expressive implicit regularization, extending the energies associated with standard hypergraph diffusion. This novel architecture generalizes classical Hypergraph Networks, and we experimentally show that it outperform existing methods on several datasets. Our technique of replacing the hypergraph Laplacian with a sheaf hypergraph Laplacian in both HyperGNN and HyperGCN establishes a versatile framework that can be employed to "sheafify" other hypergraph architectures. 
We believe that sheaf hypergraphs can contribute to further advancements in the rapidly evolving hypergraph community, extending far beyond the results presented in this work.

\paragraph{Contribution Statement} I.D. developed the models, derived all theoretical results, contributed to the design and execution of the experiments and to the writing of the manuscript. G.C. assisted with dataset preparation and experiment execution. P.L. and F.S. supervised the project and provided critical feedback. All authors contributed to refining the manuscript.

\paragraph{Acknowledgment}
The authors would like to thank Ferenc Husz\'{a}r for fruitful discussions and constructive suggestions during the development of the paper and Eirik Fladmark and Laura Brinkholm Justesen for fixing a minor issue in the original HyperGCN code, which led to improved results in the baselines. Iulia Duta is a PhD student funded by a Twitter scholarship. This work was also supported by PNRR MUR projects PE0000013-FAIR, SERICS (PE00000014), Sapienza Project FedSSL, and IR0000013-SoBigData.it.

\bibliographystyle{unsrt}
\bibliography{arxiv}

\begin{thebibliography}{10}

\bibitem{pred_pat_outcome}
Catherine Tong, Emma Rocheteau, Petar Veličković, Nicholas Lane, and Pietro
  Lio.
\newblock {\em Predicting Patient Outcomes with Graph Representation Learning},
  pages 281--293.
\newblock 01 2022.

\bibitem{huang2020skipgnn_molecular_interactions}
Kexin Huang, Cao Xiao, Lucas~M Glass, Marinka Zitnik, and Jimeng Sun.
\newblock Skipgnn: predicting molecular interactions with skip-graph networks.
\newblock {\em Scientific reports}, 10(1):1--16, 2020.

\bibitem{ying2018graph_recommended}
Rex Ying, Ruining He, Kaifeng Chen, Pong Eksombatchai, William~L Hamilton, and
  Jure Leskovec.
\newblock Graph convolutional neural networks for web-scale recommender
  systems.
\newblock In {\em Proceedings of the 24th ACM SIGKDD International Conference
  on Knowledge Discovery \& Data Mining}, pages 974--983, 2018.

\bibitem{wang2018videos_gupta2}
Xiaolong Wang and Abhinav Gupta.
\newblock Videos as space-time region graphs.
\newblock In {\em Proceedings of the European Conference on Computer Vision
  (ECCV)}, pages 399--417, 2018.

\bibitem{Crossley2014TheHO}
Nicolas~A. Crossley, Andrea Mechelli, Jessica Scott, Francesco Carletti,
  Peter~T. Fox, Philip~K. McGuire, and Edward~T. Bullmore.
\newblock The hubs of the human connectome are generally implicated in the
  anatomy of brain disorders.
\newblock {\em Brain}, 137:2382 -- 2395, 2014.

\bibitem{brain_ho}
Tingting Guo, Yining Zhang, Yanfang Xue, and Lishan Qiao.
\newblock Brain function network: Higher order vs. more discrimination.
\newblock {\em Frontiers in Neuroscience}, 15, 08 2021.

\bibitem{jost2018hypergraph_chemistry}
Jürgen Jost and Raffaella Mulas.
\newblock Hypergraph laplace operators for chemical reaction networks, 2018.

\bibitem{environment_ho}
Pragya Singh and Gaurav Baruah.
\newblock Higher order interactions and species coexistence.
\newblock {\em Theoretical Ecology}, 14, 03 2021.

\bibitem{pollinators}
Alba Cervantes‐Loreto, Carolyn Ayers, Emily Dobbs, Berry Brosi, and Daniel
  Stouffer.
\newblock The context dependency of pollinator interference: How environmental
  conditions and co‐foraging species impact floral visitation.
\newblock {\em Ecology Letters}, 24, 07 2021.

\bibitem{social_ho}
Unai Alvarez-Rodriguez, Federico Battiston, Guilherme Ferraz~de Arruda, Yamir
  Moreno, Matjaž Perc, and Vito Latora.
\newblock Evolutionary dynamics of higher-order interactions in social
  networks.
\newblock {\em Nature Human Behaviour}, 5:1--10, 05 2021.

\bibitem{HyperNN}
Yifan Feng, Haoxuan You, Zizhao Zhang, Rongrong Ji, and Yue Gao.
\newblock Hypergraph neural networks.
\newblock {\em Proc. Conf. AAAI Artif. Intell.}, 33(01):3558--3565, July 2019.

\bibitem{ramon_hypergraph}
Ramon Vi{\~n}as, Chaitanya~K. Joshi, Dobrik Georgiev, Bianca Dumitrascu,
  Eric~R. Gamazon, and Pietro Li{\`o}.
\newblock Hypergraph factorisation for multi-tissue gene expression imputation.
\newblock {\em bioRxiv}, 2022.

\bibitem{hyper_altzheimer}
Xinlei Wang, Junchang Xin, Zhongyang Wang, Chuangang Li, and Zhiqiong Wang.
\newblock An evolving hypergraph convolutional network for the diagnosis of
  alzheimers disease.
\newblock {\em Diagnostics}, 12(11), 2022.

\bibitem{reidentification_hyper}
Yichao Yan, Jie Qin, Jiaxin Chen, Li~Liu, Fan Zhu, Ying Tai, and Ling Shao.
\newblock Learning multi-granular hypergraphs for video-based person
  re-identification.
\newblock In {\em 2020 {IEEE/CVF} Conference on Computer Vision and Pattern
  Recognition, {CVPR} 2020, Seattle, WA, USA, June 13-19, 2020}, pages
  2896--2905. {IEEE}, 2020.

\bibitem{HEAT}
Dobrik~Georgiev Georgiev, Marc Brockschmidt, and Miltiadis Allamanis.
\newblock {HEAT}: Hyperedge attention networks.
\newblock {\em Transactions on Machine Learning Research}, 2022.

\bibitem{hyper_traffic_pred}
Jingcheng Wang, Yong Zhang, Lixun Wang, Yongli Hu, Xinglin Piao, and Baocai
  Yin.
\newblock Multitask hypergraph convolutional networks: {A} heterogeneous
  traffic prediction framework.
\newblock {\em {IEEE} Trans. Intell. Transp. Syst.}, 23(10):18557--18567, 2022.

\bibitem{oversmoothing_hyper}
Guanzi Chen, Jiying Zhang, Xi~Xiao, and Yang Li.
\newblock Preventing over-smoothing for hypergraph neural networks.
\newblock 2022.

\bibitem{sheaf_curry}
Justin Curry.
\newblock Sheaves, cosheaves and applications, 2014.

\bibitem{nonlin_lap}
T.-H.~Hubert Chan, Anand Louis, Zhihao~Gavin Tang, and Chenzi Zhang.
\newblock Spectral properties of hypergraph laplacian and approximation
  algorithms.
\newblock {\em J. ACM}, 65(3), mar 2018.

\bibitem{monti2019fake_news}
Federico Monti, Fabrizio Frasca, Davide Eynard, Damon Mannion, and Michael~M
  Bronstein.
\newblock Fake news detection on social media using geometric deep learning.
\newblock {\em arXiv preprint arXiv:1902.06673}, 2019.

\bibitem{ijcai2018_traffic_stgcn}
Bing Yu, Haoteng Yin, and Zhanxing Zhu.
\newblock Spatio-temporal graph convolutional networks: A deep learning
  framework for traffic forecasting.
\newblock pages 3634--3640. International Joint Conferences on Artificial
  Intelligence Organization, 7 2018.

\bibitem{bodnar_sheaf_diff}
Cristian Bodnar, Francesco~Di Giovanni, Benjamin~Paul Chamberlain, Pietro
  Li{\`o}, and Michael~M. Bronstein.
\newblock Neural sheaf diffusion: A topological perspective on heterophily and
  oversmoothing in {GNN}s.
\newblock In Alice~H. Oh, Alekh Agarwal, Danielle Belgrave, and Kyunghyun Cho,
  editors, {\em Advances in Neural Information Processing Systems}, 2022.

\bibitem{sheafNN}
Jakob Hansen and Thomas Gebhart.
\newblock Sheaf neural networks.
\newblock In {\em TDA {\&} Beyond}, 2020.

\bibitem{kipf2017semi_gcn}
Thomas~N. Kipf and Max Welling.
\newblock Semi-supervised classification with graph convolutional networks.
\newblock In {\em International Conference on Learning Representations (ICLR)},
  2017.

\bibitem{defferrard2016convolutional}
Micha{\"e}l Defferrard, Xavier Bresson, and Pierre Vandergheynst.
\newblock Convolutional neural networks on graphs with fast localized spectral
  filtering.
\newblock In {\em Advances in neural information processing systems}, pages
  3844--3852, 2016.

\bibitem{DBLP:journals/corr/BrunaZSL13_bruna_lecun_spectral}
Joan Bruna, Wojciech Zaremba, Arthur Szlam, and Yann LeCun.
\newblock Spectral networks and locally connected networks on graphs.
\newblock {\em CoRR}, abs/1312.6203, 2013.

\bibitem{hansen_sheaf_lap}
Jakob Hansen and Robert Ghrist.
\newblock Toward a spectral theory of cellular sheaves.
\newblock {\em Journal of Applied and Computational Topology}, 3(4):315--358,
  aug 2019.

\bibitem{sheaf_conn_laplacian}
Federico Barbero, Cristian Bodnar, Haitz Sáez~de Ocáriz~Borde, Michael
  Bronstein, Petar Veličković, and Pietro Lio.
\newblock Sheaf neural networks with connection laplacians.
\newblock 06 2022.

\bibitem{sheaf_att}
Federico Barbero, Cristian Bodnar, Haitz~S{\'a}ez de~Oc{\'a}riz~Borde, and
  Pietro Lio.
\newblock Sheaf attention networks.
\newblock In {\em NeurIPS 2022 Workshop on Symmetry and Geometry in Neural
  Representations}, 2022.

\bibitem{sheaf_wave}
Julian Suk, Lorenzo Giusti, Tamir Hemo, Miguel Lopez, Konstantinos Barmpas, and
  Cristian Bodnar.
\newblock Surfing on the neural sheaf.
\newblock In {\em NeurIPS 2022 Workshop on Symmetry and Geometry in Neural
  Representations}, 2022.

\bibitem{purificato2023}
Antonio Purificato, Giulia Cassarà, Pietro Liò, and Fabrizio Silvestri.
\newblock Sheaf neural networks for graph-based recommender systems, 2023.

\bibitem{schlichtkrull2017modeling}
Michael Schlichtkrull, Thomas~N. Kipf, Peter Bloem, Rianne van~den Berg, Ivan
  Titov, and Max Welling.
\newblock Modeling relational data with graph convolutional networks, 2017.

\bibitem{brain_conectome}
Chen Zu, Yue Gao, Brent Munsell, Minjeong Kim, Ziwen Peng, Yingying Zhu, Wei
  Gao, Daoqiang Zhang, Dinggang Shen, and Guorong Wu.
\newblock Identifying high order brain connectome biomarkers via learning on
  hypergraph.
\newblock {\em Mach Learn Med Imaging}, 10019:1--9, October 2016.

\bibitem{rxn_hypergraph}
Gregor Urban, Christophe~N. Magnan, and Pierre Baldi.
\newblock Sspro/accpro 6: almost perfect prediction of protein secondary
  structure and relative solvent accessibility using profiles, deep learning
  and structural similarity.
\newblock {\em Bioinform.}, 38(7):2064--2065, 2022.

\bibitem{fish_schools}
Yoshinobu Inada and Keiji Kawachi.
\newblock Order and flexibility in the motion of fish schools.
\newblock {\em Journal of theoretical biology}, 214:371--87, 03 2002.

\bibitem{hnhn}
Yihe Dong, Will Sawin, and Yoshua Bengio.
\newblock Hnhn: Hypergraph networks with hyperedge neurons.
\newblock In {\em Graph Representation Learning and Beyond Workshop at ICML
  2020}, June 2020.
\newblock Code available: https://github.com/twistedcubic/HNHN.

\bibitem{HyperGCN}
Naganand Yadati, Madhav Nimishakavi, Prateek Yadav, Vikram Nitin, Anand Louis,
  and Partha Talukdar.
\newblock Hypergcn: A new method for training graph convolutional networks on
  hypergraphs.
\newblock In H.~Wallach, H.~Larochelle, A.~Beygelzimer, F.~d\textquotesingle
  Alch\'{e}-Buc, E.~Fox, and R.~Garnett, editors, {\em Advances in Neural
  Information Processing Systems}, volume~32. Curran Associates, Inc., 2019.

\bibitem{HCHA}
Song Bai, Feihu Zhang, and Philip~H.S. Torr.
\newblock Hypergraph convolution and hypergraph attention.
\newblock {\em Pattern Recognition}, 110:107637, 2021.

\bibitem{HERALD}
Jiying Zhang, Yuzhao Chen, Xiong Xiao, Runiu Lu, and Shutao Xia.
\newblock Learnable hypergraph laplacian for hypergraph learning.
\newblock In {\em ICASSP}, 2022.

\bibitem{transformer}
Ashish Vaswani, Noam Shazeer, Niki Parmar, Jakob Uszkoreit, Llion Jones,
  Aidan~N Gomez, \L~ukasz Kaiser, and Illia Polosukhin.
\newblock Attention is all you need.
\newblock In I.~Guyon, U.~Von Luxburg, S.~Bengio, H.~Wallach, R.~Fergus,
  S.~Vishwanathan, and R.~Garnett, editors, {\em Advances in Neural Information
  Processing Systems}, volume~30. Curran Associates, Inc., 2017.

\bibitem{UniGNN}
Jing Huang and Jie Yang.
\newblock Unignn: a unified framework for graph and hypergraph neural networks.
\newblock In {\em Proceedings of the Thirtieth International Joint Conference
  on Artificial Intelligence, {IJCAI-21}}, 2021.

\bibitem{allset}
Eli Chien, Chao Pan, Jianhao Peng, and Olgica Milenkovic.
\newblock You are allset: A multiset function framework for hypergraph neural
  networks.
\newblock In {\em International Conference on Learning Representations}, 2022.

\bibitem{DeepSets}
Manzil Zaheer, Satwik Kottur, Siamak Ravanbakhsh, Barnabas Poczos, Russ~R
  Salakhutdinov, and Alexander~J Smola.
\newblock Deep sets.
\newblock In I.~Guyon, U.~Von Luxburg, S.~Bengio, H.~Wallach, R.~Fergus,
  S.~Vishwanathan, and R.~Garnett, editors, {\em Advances in Neural Information
  Processing Systems}, volume~30. Curran Associates, Inc., 2017.

\bibitem{hyper_sage}
Devanshu Arya, Deepak~K. Gupta, Stevan Rudinac, and Marcel Worring.
\newblock Hypersage: Generalizing inductive representation learning on
  hypergraphs, 2020.

\bibitem{hyper_att}
Song Bai, Feihu Zhang, and Philip H.~S. Torr.
\newblock Hypergraph convolution and hypergraph attention.
\newblock {\em CoRR}, abs/1901.08150, 2019.

\bibitem{opinion_dyn}
Jakob Hansen and Robert Ghrist.
\newblock Opinion dynamics on discourse sheaves.
\newblock {\em SIAM Journal on Applied Mathematics}, 81(5):2033--2060, 2021.

\bibitem{hol_graphs}
Sameer Agarwal, Kristin Branson, and Serge Belongie.
\newblock Higher order learning with graphs.
\newblock volume 2006, pages 17--24, 01 2006.

\bibitem{total_var_hyper}
Matthias Hein, Simon Setzer, Leonardo Jost, and Syama~Sundar Rangapuram.
\newblock The total variation on hypergraphs - learning on hypergraphs
  revisited.
\newblock In C.J. Burges, L.~Bottou, M.~Welling, Z.~Ghahramani, and K.Q.
  Weinberger, editors, {\em Advances in Neural Information Processing Systems},
  volume~26. Curran Associates, Inc., 2013.

\bibitem{pmlr-v70-zhang17d}
Chenzi Zhang, Shuguang Hu, Zhihao~Gavin Tang, and T-H.~Hubert Chan.
\newblock Re-revisiting learning on hypergraphs: Confidence interval and
  subgradient method.
\newblock In Doina Precup and Yee~Whye Teh, editors, {\em Proceedings of the
  34th International Conference on Machine Learning}, volume~70 of {\em
  Proceedings of Machine Learning Research}, pages 4026--4034. PMLR, 06--11 Aug
  2017.

\bibitem{wang2022equivariant}
Peihao Wang, Shenghao Yang, Yunyu Liu, Zhangyang Wang, and Pan Li.
\newblock Equivariant hypergraph diffusion neural operators.
\newblock {\em arXiv preprint arXiv:2207.06680}, 2022.

\bibitem{Saito_2018}
Shota Saito, Danilo Mandic, and Hideyuki Suzuki.
\newblock Hypergraph p-laplacian: A differential geometry view.
\newblock {\em Proceedings of the {AAAI} Conference on Artificial
  Intelligence}, 32(1), apr 2018.

\bibitem{house_data}
Philip~S. Chodrow, Nate Veldt, and Austin~R. Benson.
\newblock Generative hypergraph clustering: From blockmodels to modularity.
\newblock {\em Science Advances}, 7(28).

\bibitem{senate_congress_data}
James~H. Fowler.
\newblock Legislative cosponsorship networks in the {US} house and senate.
\newblock {\em Social Networks}, 28(4):454--465, oct 2006.

\bibitem{contextual_block_model}
Yash Deshpande, Subhabrata Sen, Andrea Montanari, and Elchanan Mossel.
\newblock Contextual stochastic block models.
\newblock In S.~Bengio, H.~Wallach, H.~Larochelle, K.~Grauman, N.~Cesa-Bianchi,
  and R.~Garnett, editors, {\em Advances in Neural Information Processing
  Systems}, volume~31. Curran Associates, Inc., 2018.

\bibitem{contextual_block_model2}
Debarghya Ghoshdastidar and Ambedkar Dukkipati.
\newblock Consistency of spectral partitioning of uniform hypergraphs under
  planted partition model.
\newblock In Z.~Ghahramani, M.~Welling, C.~Cortes, N.~Lawrence, and K.Q.
  Weinberger, editors, {\em Advances in Neural Information Processing Systems},
  volume~27. Curran Associates, Inc., 2014.

\bibitem{contextual_block_model3}
I~Chien, Chung-Yi Lin, and I-Hsiang Wang.
\newblock Community detection in hypergraphs: Optimal statistical limit and
  efficient algorithms.
\newblock In Amos Storkey and Fernando Perez-Cruz, editors, {\em Proceedings of
  the Twenty-First International Conference on Artificial Intelligence and
  Statistics}, volume~84 of {\em Proceedings of Machine Learning Research},
  pages 871--879. PMLR, 09--11 Apr 2018.

\bibitem{ijcai21-UniGNN}
Jing Huang and Jie Yang.
\newblock Unignn: a unified framework for graph and hypergraph neural networks.
\newblock In {\em Proceedings of the Thirtieth International Joint Conference
  on Artificial Intelligence, {IJCAI-21}}, 2021.

\bibitem{hyperND}
Francesco Tudisco, Austin~R. Benson, and Konstantin Prokopchik.
\newblock Nonlinear higher-order label spreading.
\newblock In {\em Proceedings of the Web Conference 2021}, WWW '21, page
  2402–2413, New York, NY, USA, 2021. Association for Computing Machinery.

\bibitem{oversmoothing_graph}
Xinyi Wu, Zhengdao Chen, William~Wei Wang, and Ali Jadbabaie.
\newblock A non-asymptotic analysis of oversmoothing in graph neural networks.
\newblock In {\em The Eleventh International Conference on Learning
  Representations}, 2023.

\bibitem{bound_eigenvalues}
Leyou Xu and Bo~Zhou.
\newblock Normalized laplacian eigenvalues of hypergraphs, 2023.

\bibitem{pytorch}
Adam Paszke, Sam Gross, Francisco Massa, Adam Lerer, James Bradbury, Gregory
  Chanan, Trevor Killeen, Zeming Lin, Natalia Gimelshein, Luca Antiga, Alban
  Desmaison, Andreas K{\"{o}}pf, Edward~Z. Yang, Zach DeVito, Martin Raison,
  Alykhan Tejani, Sasank Chilamkurthy, Benoit Steiner, Lu~Fang, Junjie Bai, and
  Soumith Chintala.
\newblock Pytorch: An imperative style, high-performance deep learning library.
\newblock {\em CoRR}, abs/1912.01703, 2019.

\bibitem{kingma2014adam}
Diederik~P Kingma and Jimmy Ba.
\newblock Adam: A method for stochastic optimization.
\newblock {\em arXiv preprint arXiv:1412.6980}, 2014.

\bibitem{hua2022highorder}
Chenqing Hua, Guillaume Rabusseau, and Jian Tang.
\newblock High-order pooling for graph neural networks with tensor
  decomposition, 2022.

\bibitem{nonlin_lap_mediators}
T.-H Chan and Zhibin Liang.
\newblock {\em Generalizing the Hypergraph Laplacian via a Diffusion Process
  with Mediators}, pages 441--453.
\newblock 06 2018.

\end{thebibliography}



%
\appendix



\newpage

\section*{\centering\LARGE\bf Appendix: Sheaf Hypergraph Networks}\vspace{12mm}

In this appendix, we delve into various aspects related to the methodology, including broader impact and potential limitations, proofs for the theoretical results, technical details of the models, and additional experiments mentioned in the main paper. The code associated with the paper will be released soon.

\begin{itemize}
    \item \textbf{Section A} explores the social impact and limitations of our approach, along with a discussion on potential future improvements for the model. 
    \item \textbf{Section B} provides proofs for the two primary results presented in the main paper.
    \item \textbf{Section C} offers further information about the method and its training process.
    \item \textbf{Section D} presents additional ablation studies concerning the performance of SheafHyperGCN as the depth increases. We also provide a comprehensive version of the synthetic heterophilic experiments featured in Table $3$ of the main paper, including the standard deviation for all experiments.
\end{itemize}

\section{Broader Impact \& Limitations}

In this paper we introduce a framework that enhances hypergraphs with additional structure, called cellular hypergraph sheaf, and examine the benefits that arise from this approach. Both theoretically and empirically, we demonstrate that this design choice results in a more expressive method for processing higher-order relations compared to its counterpart. Our model is designed as a generic method for higher-order processing without specific components tailored for particular tasks. We test our model on standard benchmark datasets previously used in the literature, as well as synthetically generated datasets created for academic purposes. Consequently, we believe that our paper does not contribute to any specific negative social impacts compared to other hypergraph models.

Our objective is to better incorporate the existing structure in the data. While higher-order connectivity is provided to us in the form of hypergraph structure, we lack access to a ground-truth sheaf associated with the data. Our experiments suggest that jointly learning this structure with the classification task yields top results on several benchmarks. However, without a ground-truth object, it is difficult to determine if the optimisation process results in the ideal sheaf structure. This phenomenon is evident in our experiments comparing different types of restriction maps. Although theoretically less expressive, the diagonal restriction maps generally outperform the low-rank and general ones. This can be attributed to two potential causes: 1) a simpler, diagonal sheaf may be sufficiently complex for our downstream task, or 2) a more easily optimized sheaf predictor needs to be developed to further enhance performance. It is worth noting that similar behavior has been observed in literature for graph-based sheaf models, with the diagonal sheaf consistently achieving good results.

Furthermore, our predictor heavily depends on the quality of the node and hyperedge features. The benchmark datasets we use do not provide hyperedge features, necessitating their inference based on node features. We believe that developing more effective methods for extracting hypergraph features, along with improved techniques for learning the associated sheaf structure, can unlock the full potential of our model and yield even better results.

Our theoretical framework focuses solely on characterizing the expressivity, demonstrating that SheafHNN can model a broader range of functions than classical HNN. It is worth noting, however, that having a more expressive model does not always ensure better generalizability on test data. Despite this, our empirical results consistently show an advantage of using SheafHNN over HNN across all the tested datasets, which hints that the model can perform well in terms of generalization. Nonetheless, we currently lack any theoretical analyses to substantiate this conclusion.






\begin{figure}[t!]
    \centering
     \vspace{12mm}
    \includegraphics[scale=0.22]{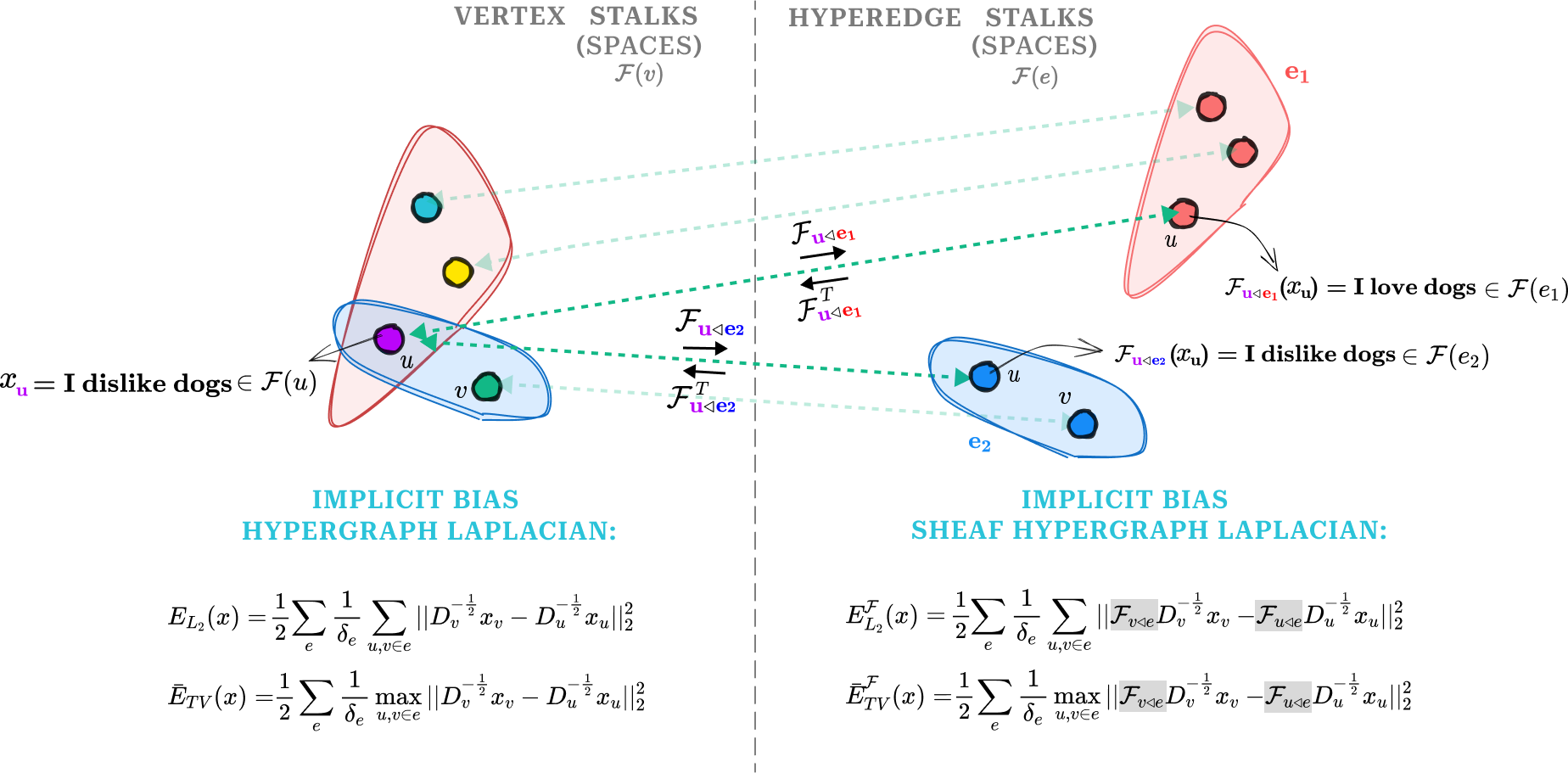}
    \caption{\textbf{Visualisation of a hypergraph sheaf and the inductive biases associated with its Laplacian.} For a hypergraph, each node has associated a vertex stalk $F(v)=\mathbb{R}^d$, and each hyperedge has associated a hyperedge stalk $F(e)=\mathbb{R}^d$. For each incident pair $(v, e)$  we can move from the vertex stalk $F(v)$ to the hyperedge stalk $F(e)$ via a linear map $\mathcal{F}_{v \triangleleft e}: \mathbb{R}^d \rightarrow \mathbb{R}^d $. 
     Hypergraph Networks implicitly minimize an energy function defined on the vertex space, aiming to bring together the representation of the neighbouring nodes (\textbf{left}). Differently, our Sheaf Hypergraph Networks aim to reduce the discrepancy between the neighbouring representations in the hyperedge space (\textbf{right}). This has several advantages. Firstly, we prevent the features from becoming uniform by minimizing the distance in a more complex space, as confirmed both theoretically and empirically. Moreover, in the hyperedge space, each node can have a different representation for each hyperedge it is part of, leading to a more expressive and general message passing framework.
    }
    \label{fig:main_fig}
    \vspace{5mm}
\end{figure} 

\section{Proofs}

\subsection{Proof of Proposition 1}
\begin{proposition}[]
The diffusion process using a symmetric normalised version of the linear sheaf hypergraph Laplacian minimizes the sheaf Dirichlet energy of a signal $x$ on a hypergraph $\mathcal{H}$. Moreover, the energy decreases with each layer of diffusion. 
\end{proposition}
\begin{proof}

We will first demonstrate that a single layer of the diffusion process $Y = (I - \Delta^{\mathcal{F}})X$ using symmetric normalised version of the linear sheaf hypergraph Laplacian $\Delta^{\mathcal{F}} = D^{-\frac{1}{2}}\mathcal{L^F}D^{-\frac{1}{2}}  \in \mathbb{R}^{nd \times nd}$ reduces the sheaf Dirichlet energy as follow: $E^{\mathcal{F}}_{L_2}(Y) \leq \lambda_{*}E^{\mathcal{F}}_{L_2}(X)$, with $\{\lambda_i\}$ the eigenvalues of $\Delta^{\mathcal{F}}$ and $\lambda_* = \max_{i; \lambda_i \neq 0}{\{(1-\lambda_i)^2\}} < 1$. In this manner, as the number of layers approaches infinity, the energy will converge to the minimum energy of $0$.

To prove that $E^{\mathcal{F}}_{L_2}(Y) \leq \lambda_{*}E^{\mathcal{F}}_{L_2}(X)$ we employ the same technique used to demonstrate that the energy associated with the standard hypergraph diffusion decreases with each layer~\cite{oversmoothing_hyper}. We will first rewrite the energy as a quadratic term, then decompose it in the eigenvector space and bound it using the eigenvalues upper bounds.

We initiate our proof by establishing a series of Lemmas that will be utilized in our argument.
\begin{lemma} \label{dirich_rephrase} $E^{\mathcal{F}}_{L_2}(x) =  \frac{1}{2}\sum_e \frac{1}{\delta_e}\sum_{u,v \in e}||\mathcal{F}_{v \tri e}D_v^{-\frac{1}{2}}x_v -\mathcal{F}_{u \tri e}D_u^{-\frac{1}{2}}x_u||_2^2 = x^T \Delta^{\mathcal{F}} x $
\end{lemma}
\begin{proof}
    \begin{align*}
x^T\Delta^{\mathcal{F}}x &= x^TD^{-\frac{1}{2}}\mathcal{L^F}D^{-\frac{1}{2}}x \\
&= \sum_e \frac{1}{\delta_e}\Big( \sum_{v \in e} x_v^T D_v^{-\frac{1}{2}} \mathcal{F}^T_{v \tri e}  \mathcal{F}_{v \tri e} D_v^{-\frac{1}{2}} x_v 
- 
\sum_{\substack{w,z \in e\\ w \neq z}} x_w^T D_w^{-\frac{1}{2}} \mathcal{F}^T_{w \tri e}  \mathcal{F}_{z \tri e} D_z^{-\frac{1}{2}} x_z \Big) \\
&= \frac{1}{2}\sum_e \frac{1}{\delta_e}\Big( \sum_{w \in e} x_w^T D_w^{-\frac{1}{2}} \mathcal{F}^T_{w \tri e}   \mathcal{F}_{w \tri e} D_w^{-\frac{1}{2}} x_w 
+ \sum_{z \in e} x_z^T D_z^{-\frac{1}{2}} \mathcal{F}^T_{z \tri e} \mathcal{F}_{z \tri e} D_z^{-\frac{1}{2}} x_z \\
&\hspace{10mm}- \sum_{\substack{w,z \in e\\ w \neq z}} x_w^T D_w^{-\frac{1}{2}} \mathcal{F}^T_{w \tri e}  \mathcal{F}_{z \tri e} D_z^{-\frac{1}{2}} x_z
- \sum_{\substack{w,z \in e\\ w \neq z}} x_z^T D_z^{-\frac{1}{2}} \mathcal{F}^T_{z \tri e} D_e^{-1}  \mathcal{F}_{w \tri e} D_w^{-\frac{1}{2}} x_w \Big) 
\end{align*}

\begin{align*}
&\hspace{-10mm} = \frac{1}{2}\sum_e \frac{1}{\delta_e}\sum_{w,z \in e} \Big(x_w^T D_w^{-\frac{1}{2}} \mathcal{F}^T_{w \tri e}- x_z^T D_z^{-\frac{1}{2}} \mathcal{F}^T_{z \tri e} \Big) \Big( \mathcal{F}_{w \tri e} D_w^{-\frac{1}{2}}x_w
-   \mathcal{F}^T_{z \tri e} D_z^{-\frac{1}{2}} x_z \Big) \\
&\hspace{-10mm} =\frac{1}{2}\sum_e \frac{1}{\delta_e}\sum_{w,z \in e} \Big(   \mathcal{F}_{w \tri e} D_w^{-\frac{1}{2}}x_w
-\mathcal{F}_{z \tri e} D_z^{-\frac{1}{2}} x_z \Big)^T\Big(   \mathcal{F}_{w \tri e} D_w^{-\frac{1}{2}}x_w
-  \mathcal{F}_{z \tri e} D_z^{-\frac{1}{2}} x_z \Big) \\
&\hspace{-10mm} = \frac{1}{2}\sum_e \frac{1}{\delta_e}\sum_{w,z \in e} ||(\mathcal{F}_{v \tri e}D_v^{-\frac{1}{2}}x_v -\mathcal{F}_{u \tri e}D_u^{-\frac{1}{2}}x_u)||_2^2
\end{align*}
\end{proof}

\begin{lemma} \label{lemma_eigenval} The eigenvalues of the symmetric normalised linear sheaf Laplacian $\Delta^{\mathcal{F}}$ are in $[0,1]$.
\end{lemma}

\begin{proof}
Let's denote by $\{\lambda_i\}$ the set of eigenvalues of $\Delta^{\mathcal{F}} \in \mathbb{R}^{nd \times nd}$. We want to prove that $\lambda_i \in [0,1]$.
Let's consider $\lambda_1 \geq \lambda_2 \dots \lambda_{nd}$.

\begin{align*}
    \lambda_1 &= \max_{x} \frac{<x, \Delta^{\mathcal{F}} x>}{<x,x>} \\
                &= \max_x \frac{\frac{1}{2}\sum_{e}{\frac{1}{\delta_e}} \sum_{u,v \in e} ||\mathcal{F}_{v \triangleleft e} D_v^{-\frac{1}{2}} x_v - \mathcal{F}_{u \triangleleft e} D_u^{-\frac{1}{2}} x_u||_2^2}{<x,x>} \\
                &=\max_x \frac{\sum_{e}{\frac{1}{\delta_e}}(\sum_v || \mathcal{F}_{v \triangleleft e} D_v^{-\frac{1}{2}} x_v ||_2^2 - \sum_ {u,v \in e; u \neq v} <\mathcal{F}_{v \triangleleft e} D_v^{-\frac{1}{2}} x_v, \mathcal{F}_{u \triangleleft e} D_u^{-\frac{1}{2}} x_u>)}{<x,x>} \\ 
                &\leq \max_x  \frac{\sum_{e}{\frac{1}{\delta_e}}(\sum_v || \mathcal{F}_{v \triangleleft e} D_v^{-\frac{1}{2}} x_v ||_2^2 + \frac{1}{2} \sum_ {u,v \in e; u \neq v} (||\mathcal{F}_{v \triangleleft e} D_v^{-\frac{1}{2}} x_v||_2^2 + || \mathcal{F}_{u \triangleleft e} D_u^{-\frac{1}{2}} x_u ||_2^2))}{<x,x>} \\ 
                &\leq  \max_x  \frac{\sum_{e}\frac{1}{\delta_e}(\sum_v || \mathcal{F}_{v \triangleleft e} D_v^{-\frac{1}{2}} x_v ||_2^2 + \frac{1}{2} 2 (\delta_e-1)\sum_v ||\mathcal{F}_{v \triangleleft e} D_v^{-\frac{1}{2}} x_v||_2^2)}{<x,x>} \\
                &\leq \max_x  \frac{\sum_{v}\sum_e x_v^T D_v^{-\frac{1}{2}} \mathcal{F}_{v \triangleleft e}^T \mathcal{F}_{v \triangleleft e} D_v^{-\frac{1}{2}} x_v}{<x,x>} \\
                &\leq \max_x  \frac{\sum_{v}x_v^T D_v^{-\frac{1}{2}} \sum_e(\mathcal{F}_{v \triangleleft e}^T \mathcal{F}_{v \triangleleft e}) D_v^{-\frac{1}{2}} x_v }{<x,x>} \\
                &\leq \max_x  \frac{\sum_{v}x_v^T x_v }{<x,x>} = 1
\end{align*}
Since $\lambda_k = \frac{<v_k, \Delta^{\mathcal{F}} v_k>}{<v_k,v_k>} \geq 0$ we conclude that $1 \geq \lambda_1 \geq \lambda_2 .. \geq \lambda_{nd} \geq 0$ 
\end{proof}
The Lemma above generalises the result characterising the spectrum of hypergraph Laplacian~\cite{bound_eigenvalues} (Theorem 3.2), extending it for the spectrum of sheaf hypergraph Laplacian.

We will now head to prove the main result: $E^{\mathcal{F}}_{L_2}(Y) \leq \lambda_{*}E^{\mathcal{F}}_{L_2}(X)$

    Let's consider $(\lambda_i, v_i)$ the eigenvalue-eigenvector pairs for $\Delta^{\mathcal{F}}$.
    We can decomposed the hypergraph signal $x \in \mathbb{R}^{nd \times 1}$ in the eigenvector basis as $x=\sum_i c_i v_i$ for some coefficients $c_i$. 
    
    From Lemma~\ref{dirich_rephrase} we have that $E^{\mathcal{F}}_{L_2}(x) = x^T  \Delta^{\mathcal{F}} x$ so we can further decomposed it in the eigenvector basis as follow:
    
    \begin{align*}
    E^{\mathcal{F}}_{L_2}(x) &= x^T  \Delta^{\mathcal{F}} x = x^T \sum_i c_i \lambda_i v_i \\
        &= \sum_i c_i^2 \lambda_i v_i
    \end{align*}

    On the other hand we have that $E^{\mathcal{F}}_{L_2}((I - \Delta^{\mathcal{F}})x) = x^T  (I - \Delta^{\mathcal{F}})^T \Delta^{\mathcal{F}} (I - \Delta^{\mathcal{F}}) x$
    
    $\Delta^{\mathcal{F}}$ has $(\lambda_i, v_i)$ as eigenvalues-eigenvectors pairs, thus $(I - \Delta^{\mathcal{F}})$  will have the same set of eigenvectors $v_i$ with corresponding eigenvalues $(1-\lambda_i)$. Thus we can decomposed the sheaf Dirichlet energy as:

    \begin{align*}
        E^{\mathcal{F}}_{L_2}((I - \Delta^{\mathcal{F}})x) &= x^T  (I - \Delta^{\mathcal{F}})^T \Delta^{\mathcal{F}} (I - \Delta^{\mathcal{F}}) x \\
            &= \sum_i c_i^2 \lambda_i (1-\lambda_i)^2 v_i
    \end{align*}

    Let's denote by $\lambda_* = max_{i; \lambda_i \neq 0}\{(1-\lambda_i)^2\} <1$. Then:
    \begin{align*}
        \label{ineq}
        E^{\mathcal{F}}_{L_2}((I - \Delta^{\mathcal{F}})x) &= \sum_i c_i^2 \lambda_i (1-\lambda_i)^2 v_i \\
        & \leq \lambda_* \sum_i c_i^2 \lambda_i v_i = \lambda_* E^{\mathcal{F}}_{L_2}(x) 
    \end{align*}

    From Lemma~\ref{lemma_eigenval} we have that $\lambda_i \in [0,1] \Rightarrow \lambda_* < 1$ 
    
    We conclude that:
    
    \begin{equation*}
    E^{\mathcal{F}}_{L_2}((I - \Delta^{\mathcal{F}})x) \leq \lambda_*E^{\mathcal{F}}_{L_2}(x) < E^{\mathcal{F}}_{L_2}(x)
    \end{equation*}

\end{proof}

\subsection{Proof of Proposition 2}

\begin{proposition}[]
 The diffusion process using symmetric normalised version of non-linear sheaf hypergraph Laplacian minimizes the sheaf total-variance of a signal $x$ on hypergraph $\mathcal{H}$. 
\end{proposition}
\begin{proof}
Our proof is inspired by the proof in ~\cite{pmlr-v70-zhang17d}, which demonstrates that the non-linear hypergraph diffusion minimizes the total variation.

The summary of our proof is as follow: 1) we show that for a given sheaf $\mathcal{F}$ associated with a hypergraph, the sheaf total variation $\bar E_{TV}(x)$ is a convex function, and 2) the sheaf non-linear Laplacian represents a subgradient of the sheaf total variance. Since the subgradient method minimizes convex functions, we conclude that the non-linear sheaf diffusion minimizes the sheaf total variation. We explain the steps in more detail below.

The sheaf total variance is defined as follow: $\bar E^{\mathcal{F}}_{TV}(x) =  \frac{1}{2}\sum_e\frac{1}{\delta_e}\max_{u,v \in e}||{\mathcal{F}_{v \tri e}}D_v^{-\frac{1}{2}}x_v -{\mathcal{F}_{u \tri e}}D_u^{-\frac{1}{2}}x_u||_2^2$
where $D_v = \sum\limits_{e; v \in e} \mathcal{F}_{v \trianglelefteq e}^T\mathcal{F}_{v \trianglelefteq e}$.

Given a sheaf structure $\mathcal{F}$ associated with the hypergraph $\mathcal{H}$, let's denote by $g_{\mathcal{F}}(x_u, x_v) = ||\mathcal{F}_{u \tri e}D_u^{-\frac{1}{2}}x_u -\mathcal{F}_{v \tri e}D_v^{-\frac{1}{2}}x_v||$. From the triangle inequality we have that $g_{\mathcal{F}}(x_u, x_v)$ is convex since:
\begin{align*}
||{\mathcal{F}_{v \tri e}}D_v^{-\frac{1}{2}}(\theta x_v + (1-\theta) y_v) -{\mathcal{F}_{u \tri e}}D_u^{-\frac{1}{2}}&(\theta x_u + (1-\theta) y_u)|| \leq \theta ||{\mathcal{F}_{v \tri e}}D_v^{-\frac{1}{2}}x_v -{\mathcal{F}_{u \tri e}}D_u^{-\frac{1}{2}}x_u|| \\
&+ (1-\theta)||{\mathcal{F}_{v \tri e}}D_v^{-\frac{1}{2}}y_v -{\mathcal{F}_{u \tri e}}D_u^{-\frac{1}{2}}y_u|| \\
\end{align*}

The square of a convex positive function is convex. The maximum of a set of convex functions is convex. Thus: $\max_{w,z}\{g_{\mathcal{F}}(x_w, x_z))^2\} = \max_{w,z}||\mathcal{F}_{w \tri e}D_w^{-\frac{1}{2}}x_w -\mathcal{F}_{z \tri e}D_z^{-\frac{1}{2}}x_z||_2^2$ is convex. Following the same approach for all hyperedges we have that $\bar E^{\mathcal{F}}_{TV}(x) =  \frac{1}{2}\sum_e\frac{1}{\delta_e}\max_{u,v \in e}||{\mathcal{F}_{v \tri e}}D_v^{-\frac{1}{2}}x_v -{\mathcal{F}_{u \tri e}}D_u^{-\frac{1}{2}}x_u||_2^2$ is also a convex function. 

$\bar E_{TV}(x)$ is a convex, but non-differentiable function. Therefore, it can be minimized using the subgradient method. In the following, we will demonstrate that the symmetric normalized non-linear sheaf Laplacian operator is a subgradient for $\bar E_{TV}(x)$.

It is straightforward to establish that:
\begin{align*}
\bar E_{TV}(x) = \frac{1}{2}x^T \bar \Delta^{\mathcal{F}}x
\end{align*}

\begin{align*}
x^T\bar\Delta^{\mathcal{F}}x &= x^TD^{-\frac{1}{2}}\mathcal{\bar L^F}D^{-\frac{1}{2}}x \\
&= \sum_{e; w \sim_e z} \frac{1}{\delta_e}\Big( x_w^T D_w^{-\frac{1}{2}} \mathcal{F}^T_{w \tri e}  \mathcal{F}_{w \tri e} D_w^{-\frac{1}{2}} x_w 
- 
 x_w^T D_w^{-\frac{1}{2}} \mathcal{F}^T_{w \tri e}  \mathcal{F}_{z \tri e} D_z^{-\frac{1}{2}} x_z\Big)  \\
&=  \frac{1}{2}\sum_{e;  w \sim_e z} \frac{1}{\delta_e}\Big(  x_w^T D_w^{-\frac{1}{2}} \mathcal{F}^T_{w \tri e}   \mathcal{F}_{w \tri e} D_w^{-\frac{1}{2}} x_w 
+ x_z^T D_z^{-\frac{1}{2}} \mathcal{F}^T_{z \tri e} \mathcal{F}_{z \tri e} D_z^{-\frac{1}{2}} x_z \\
&\hspace{10mm}-x_w^T D_w^{-\frac{1}{2}} \mathcal{F}^T_{w \tri e}  \mathcal{F}_{z \tri e} D_z^{-\frac{1}{2}} x_z
-  x_z^T D_z^{-\frac{1}{2}} \mathcal{F}^T_{z \tri e} D_e^{-1}  \mathcal{F}_{w \tri e} D_w^{-\frac{1}{2}} x_w \Big) \\
&\hspace{-10mm} = \frac{1}{2}\sum_{e; w \sim_e z} \frac{1}{\delta_e}\Big(x_w^T D_w^{-\frac{1}{2}} \mathcal{F}^T_{w \tri e}- x_z^T D_z^{-\frac{1}{2}} \mathcal{F}^T_{z \tri e} \Big) \Big( \mathcal{F}_{w \tri e} D_w^{-\frac{1}{2}}x_w
-   \mathcal{F}^T_{z \tri e} D_z^{-\frac{1}{2}} x_z \Big) \\
&\hspace{-10mm} =\frac{1}{2}\sum_{e; w \sim_e z} \frac{1}{\delta_e} \Big(   \mathcal{F}_{w \tri e} D_w^{-\frac{1}{2}}x_w
-\mathcal{F}_{z \tri e} D_z^{-\frac{1}{2}} x_z \Big)^T\Big(   \mathcal{F}_{w \tri e} D_w^{-\frac{1}{2}}x_w
-  \mathcal{F}_{z \tri e} D_z^{-\frac{1}{2}} x_z \Big) \\
&\hspace{-10mm} = \frac{1}{2}\sum_{e; w \sim_e z} \frac{1}{\delta_e}2||\mathcal{F}_{w \tri e}D_w^{-\frac{1}{2}}x_w -\mathcal{F}_{z \tri e}D_z^{-\frac{1}{2}}x_z||_2^2\\
&\hspace{-10mm} = \sum_{e} \frac{1}{\delta_e}\max_{w,z}||\mathcal{F}_{w \tri e}D_w^{-\frac{1}{2}}x_w -\mathcal{F}_{z \tri e}D_z^{-\frac{1}{2}}x_z||_2^2\\
&\hspace{-10mm}=2\bar E_{TV}(x)
\end{align*}

When the distances $g_{\mathcal{F}}(x_w, x_z)$ are distinct for each pair $w,z \in e$, the function $\bar E_{TV}(x)$ is differentiable and $\bar\Delta^{\mathcal{F}}x$ represent its gradient (since it is the gradient of the quadratic form $\frac{1}{2}x^T\bar\Delta^{\mathcal{F}}x$). So, we only need to find the subgradient for the points where the function is not differentiable (the points where, for a hyperedge, several different pairs of nodes achieve maximum distance).

It is known that when $f$ represent the maximum over a set of convex functions $f_i$, the set of all subgradients of $f$ is determined as $ \partial f(x) = \text{conv} \cup_{\alpha \in \mathcal{A}(x)} \partial f_{\alpha}(x)$ with $\mathcal{A}(x) = \{\alpha | f_{\alpha}(x)=f(x)\}$. Intuitively, for an hyperedge $e$, if we have more than one pair of nodes achieving maximum distance, computing the derivative for any of these pairs would give us a subgradient of the function.  This means that, by following our approach of breaking the ties randomly inside $\bar\Delta^{\mathcal{F}}x$ leads to a valid subgradient for those points where the function is non-differentiable.

Given that $\bar E_{TV}(x)$ is convex non-differentiable, with the non-linear sheaf diffusion operator described by $\bar\Delta^{\mathcal{F}}x=\sum_{e; u \sim_e v}\frac{1}{\delta_e}D_v^{-\frac{1}{2}}\mathcal{F}^T_{v \trianglelefteq e}\big(\mathcal{F}_{v \trianglelefteq e}D_v^{-\frac{1}{2}}x_v-\mathcal{F}_{u \trianglelefteq e}D_u^{-\frac{1}{2}}x_u \big)$ as a subgradient, we conclude that the diffusion process using the non-linear sheaf diffusion operator minimizes the total variance $\bar E_{TV}(x)$.

\end{proof}

\section{Experimental details}
\subsection{Datasets.} We are running experiments on a set of real-world benchmarking datasets and also on a synthetic set of datasets as in~\cite{wang2022equivariant}. We provide here details about both of them.

\paragraph{Real-world  datasets.}
For the real-world datasets, we test our model on Cora, Citeseer, Pubmed, Cora-CA, and DBLP-CA from~\cite{HyperGCN}, as well as House~\cite{house_data}, Senate, and Congress~\cite{senate_congress_data}. These datasets encompass various network types, including citation networks, co-authorship networks, and political affiliation networks, featuring diverse node and hyperedge features. In co-citation networks (Cora, Citeseer, Pubmed), all documents cited by a specific document are connected through a hyperedge \cite{HyperGCN}. For co-authorship networks (Cora-CA, DBLP), all documents co-authored by a specific author form a single hyperedge \cite{HyperGCN}. Node features in both citation and co-authorship networks are represented by the bag-of-words of the associated documents, while node labels correspond to paper classes.
In the House dataset, each node represents a member of the US House of Representatives, with hyperedges grouping members belonging to the same committee.
For the Congress and Senate datasets, we adhere to the same settings as described in \cite{wang2022equivariant}. In the Congress dataset, nodes symbolize US Congress persons, and hyperedges comprise the sponsor and co-sponsors of legislative bills introduced in both the House of Representatives and the Senate. In the Senate dataset, nodes once again represent US Congress persons, but hyperedges consist of the sponsor and co-sponsors of bills introduced exclusively in the Senate. Each node in both datasets is labeled with the respective political party affiliation.

While reproducing the HyperGCN baseline as reported in~\cite{allset,wang2022equivariant}, we uncovered an essential aspect related to self-loops. The weight coefficient calculation, given by $\frac{1}{2|e|-3}$, yields a value of $-1$ for self-loops when $|e| = 1$. To rectify this, we experimented with an alternative formula: $\max(1, \frac{1}{2|e|-3}$) for determining the weight coefficient, diverging from the original implementation in the  \cite{allset} repository. By incorporating this modification, we recorded substantial enhancement in HyperGCN's performance, especially within the heterophilic setup. For fairness reasons we decided to report these results as opposed to the original ones in Table 1 of the main paper.

\paragraph{Synthetic  datasets.}
We generate the synthetic heterophilic dataset inspired by the ones introduced by~\cite{wang2022equivariant}. We generate a $5000$ nodes hypergraph using the contextual hypergraph stochastic block model~\cite{contextual_block_model,contextual_block_model2,contextual_block_model3}. Half of these nodes belong to class $0$ while the other half to class $1$, and task is formulated as a node-classification problem. 
We randomly sample $1000$ hyperedges, each with a cardinality of $15$, to create the hyperedges. Each hyperedge contains exactly $\beta$ nodes from class $0$ and $15-\beta$ nodes from class $1$. The heterophily level is computed as $\alpha = \text{min}(\beta, 15-\beta$). To create node features, we sample from a label-dependent Gaussian distribution with a standard deviation of $1$. Since the original dataset is not publicly available, we generate our own set of $7$ datasets, by varying  the heterophilic level $\alpha \in \{ 1 \dots 7 \}$.

\subsection{Implementation details.}
We define a layer of Sheaf Hypergraph Network as used in both the linear (SheafHyperGNN) and non-linear (SheafHyperGCN) version of the architecture as follow:
\begin{equation} \label{eq:model}
Y = \sigma((I_{nd}-\overset{\bullet}\Delta)(I_n \otimes W_1 )\tilde XW_2)
\end{equation}

For a hypergraph with $n$ nodes, we denote by $\tilde{X} \in \mathbb{R}^{nd \times f}$ the representations of the nodes in the vertex stalks of a $d$-dimensional sheaf. We note that instead of each node being represented as a row in the feature matrix, as is common in standard HNNs, in our SheafHNN each node is characterized by $d$ rows, resulting in a $d \times f$ feature matrix for each node.

In Equation.~\ref{eq:model} the features corresponding to each node are processed in a factorized way: $W_1 \in \mathbb{R}^{d \times d}$  is used to combine the $d$ dimensions of the vertex stalk, independently for each node and each channel, while $W_2 \in \mathbb{R}^{f \times f}$ corresponds to the usual linear projection combining the  $f$ features, independently for each node and each stalk dimension. $\sigma$ denotes ReLU non-linearity.

In our codebase, for fair comparison against other baselines, we use the same pipeline from the public repository of ~\cite{allset} and ~\cite{wang2022equivariant}. We also use some auxiliary functions from~\cite{bodnar_sheaf_diff}. All models are trained in PyTorch~\cite{pytorch}, using Adam optimizer~\cite{kingma2014adam} for 100 epochs. For all sheaf-based experiments, we observed a much faster convergence than the usual HNNs. 

The results reported in Table 1 and Table 2 of the main paper are obtained using random hyper-parameter tuning. We search the following sets of hyper-parameters:

\begin{enumerate}[leftmargin=*]
\item \textit{Stack dimension} $d$ from $1-8$
\item The sheaf structure is either \textit{shared between layers or recompute} every layer based on the intermediate representations.  We noticed that the fixed one is, in general, easier to optimize.
\item $W_1$ is either a \textit{learnable parameter or fixed} to the identity matrix.
\item The \textit{type of normalisation} used for the Laplacian is either symmetric $\Delta = D^{-\frac{1}{2}}\mathcal{L}D^{-\frac{1}{2}}$ or asymmetric $\Delta = D^{-1}\mathcal{L}$; based on the degree ($D$ is the degree matrix) or sheaf-based ($D = \sum\limits_{e; v \in e} \mathcal{F}_{v \trianglelefteq e}^T\mathcal{F}_{v \trianglelefteq e}$). In our experiments the degree-based normalisation is, in general, more stable to optimize.
\item Since none of datasets we are using has \textit{hyperedge features}, we are treating the way to generate edge features ($x_e$) as a hyperparameter. We experiment with the following approaches: a) $x_e= \bigoplus_{v\in e}x_v$; b) $x_e= \bigoplus_{v\in e}h_v$, with $h_v$ the hidden representation of the nodes c) $x_e= \bigoplus_{v\in e}\text{MLP}(x_v)$ or d) more general as in~\cite{hua2022highorder} $x_e = \sigma'\Big(M \sigma\Big(W^T \begin{bmatrix} x_{v_1} \\ 1 \end{bmatrix} \odot \dots \odot \begin{bmatrix} x_{v_k} \\ 1 \end{bmatrix}\Big)\Big)$, with $\sigma'=\text{ReLU}$ and $\sigma=\text{tanh}$.  Generally, all of them behaves on par.
\item \textit{Non-linearity} for $\mathcal{F}_{v \tri e}$ is either sigmoid or tanh.
\item \textit{Learning rate} from $\{0.1, 0.01, 0.001\}$; \textit{weight decay} from $\{0, 1e-05\}$; \textit{dropout rate} from $\{0.1, 0.2 \dots 0.9\}$
\item \textit{Number of layers} from $1-8$; \textit{hidden dimension} from $\{16, 32, 64, 128, 256, 512\}$
\end{enumerate}

\subsection{Restriction maps details.} As mentioned in the main paper, the general setup for predicting restriction maps is as follow: for each incidence pair $(v, e)$, we predict a $d \times d$ block matrix $ \mathcal{F}_{v \trianglelefteq e} = \text{MLP}(x_v || h_e) \in \mathbb{R}^{d^2}$.
We experiment with three types of $d \times d$ restriction maps: diagonal, low-rank and general.

\paragraph{General restriction maps.} In the general case, the MLP predicts $d^2$ parameters that we rearrange in a square $d \times d$ matrix representing the restriction maps.

\paragraph{Diagonal restriction maps.} For the diagonal restriction maps, the MLP used to predict the restriction maps outputs only $d$ elements. The final restriction block is obtained by creating a matrix with these $d$ elements on the diagonal. While being less expressive than the general case, the diagonal matrix is more efficient both in terms of parameters and computation.

\paragraph{Low-Rank restriction maps.} In the low-rank case, to predict a rank-$r$ matrix, $2*d*r+d$ elements are predicted. They are rearranged in two matrices $A \in \mathbb{R}^{d \times r}$, $B \in \mathbb{R}^{r \times d}$ and a vector $c \in \mathbb{R}^{d \times 1}$. The final restriction matrix is obtained as $AB^T+diag(c)$. When $r < (d-1)/2$, the low-rank represent a more efficient option in terms of parameters. 

In experiments, we observed that models using the diagonal version of the restriction map consistently outperform those using either low-rank or general versions. We believe that finding better, more efficient ways of predicting and optimizing the restriction maps might further improve the results presented in this paper.

\subsection{Non-linear sheaf Laplacian with mediators.}
The non-linear hypergraph Laplacian relies on a much sparser representation compared to the linear hypergraph Laplacian. While the linear one creates $\frac{|e|(|e|-1)}{2}$ edges for each hyperedge, the non-linear one only draws a single connection for each hyperedge. To allow each node to participate in the hypergraph diffusion, while maintaining a sparse representation, ~\cite{nonlin_lap_mediators} introduces a variation of the non-linear Laplacian that propagates flow through all the nodes in a hyperedge and still preserves the theoretical properties of the original non-linear Laplacian. 

\begin{definition} Based on their formulation, we define the \textit{non-linear sheaf hypergraph Laplacian with mediators} as follow:

\begin{enumerate}
    \item For each hyperedge $e$, compute  $(u_e, v_e) = argmax_{u,v \in e} ||\mathcal{F}_{u \triangleleft e}x_u - \mathcal{F}_{v \triangleleft e}x_v||$, the set of pairs containing the nodes with the most discrepant features in the hyperedge stalk \footnote{Note that for the normalised version of the Laplacian $x_u\rightarrow D_u^{-\frac{1}{2}}x_u$} and the set of mediators $K_e=\{k\in e: k \neq u_e, k \neq v_e\}$
    \item Build an undirected graph $\mathcal{G}_H$ containing the same sets of nodes as $\mathcal{H}$  and, for each hyperedge $e$ connects the most discrepant nodes $(u, v)$, and also add a connection between $\{u,v\}$ and all the nodes in $K_e$. (from now on we will write $u \sim_e v$ if they are connected in the $\mathcal{G}_H$ graph due to the hyperedge e). If multiple pairs have the same maximum discrepancy, we will randomly choose one of them.  
    \item Define the sheaf non-linear hypergraph Laplacian as:
        \begin{equation} \label{eq_def_non_linear_sheaf_lap}
        \mathcal{\bar L_F}(x)_v = \sum_{e; u \sim_e v}\frac{1}{\delta_e}\mathcal{F}^T_{v \trianglelefteq e}\big(\mathcal{F}_{v \trianglelefteq e}x_v-\mathcal{F}_{u \trianglelefteq e}x_u \big)
        \end{equation}
\end{enumerate}
\end{definition}
We note that in the non-linear Laplacian with mediators, all the nodes in a hyperedge  have an active role in the diffusion. This approach requires a linear number of edges, remaining more efficient than the linear version of the Laplacian.

We use this approach in all the experiments involving non-linear sheaf hypergraph Laplacian.

\section{Additional experiments}

\paragraph{Full results for the comparison with baselines.} Due to the space contraints in the main paper (Table $3$), we only report the average accuracy for the experiments performed on the Synthetic heterophilic datasets (with $\alpha = 7$). In Table~\ref{tab:hetero_dataset_full} of the Appendix, we include both the average performance and the standard deviation obtained across 10 random splits. Both our models, SheafHyperGNN and SheafHyperGCN, outperform their counterparts, with SheafHyperGNN achieving the best results among all the baselines. We note that on the synthetic dataset, HyperGCN is very unstable during training, with a high standard deviation between the runs. We believe this is due to the noisy features that negatively and irreversibly affect the creation of the non-linear Laplacian. In contrast, our generalization achieves more stable results.

\begin{table}[t]
    \tiny
    \centering
    \caption{\textbf{Average accuracy and standard deviation on Synthetic Datasets with Varying Heterophily Levels}: Across all different level of heterophily, the sheaf-based methods consistently outperform their counterparts. Additionally, they achieve top results for all heterophily levels, further demonstrating their effectiveness. \\}
  \label{tab:hetero_dataset_full}
  \begin{tabular}{cccccccc}
    \toprule
    & \multicolumn{7}{c}{heterophily ($\alpha$)}\\
     Name      & $1$ &  $2$ & $3$ &  $4$ &  $5$ &  $6$ & $7$  \\
    \midrule
    HyperGCN  & 83.9 \pm 18.6  &	69.4 \pm 13.0 & 72.9 \pm 10.7 & 75.9 \pm 1.0 & 70.5 \pm 9.7 & 67.3 \pm 12.1 & 66.5 \pm 10.8\\ 
    HyperGNN  &  98.4 \pm 0.3 &  83.7 \pm 0.6 & 79.4 \pm 0.9 & 74.5 \pm 0.8 & 69.5 \pm 0.9 & 66.9 \pm 1.13 & 63.8 \pm 1.1\\ 
    HCHA  &  98.1 \pm 0.6 &  81.8 \pm 1.3 & 78.3 \pm 1.4 & 75.9 \pm 1.3 & 74.1 \pm 1.6 & 71.1 \pm 1.4 & 70.8 \pm 1.0 \\ 
    ED-HNN & 99.9 \pm 0.1 & 91.3 \pm 3.0 & 88.4 \pm 1.5 & 84.1 \pm 2.8 & 80.7 \pm 1.2 & 78.8 \pm 0.9 & 76.5 \pm 1.4 \\ 
    \midrule
    SheafHGCN  &  \textbf{100 \pm 0.0} &	87.1 \pm 2.4 & 84.8 \pm 1.1 & 79.2 \pm 2.0 & 78.1 \pm 0.6  & 76.6 \pm 1.1 & 75.5 \pm 1.4\\ 
    SheafHGNN  &  \textbf{100 \pm 0.0 } &	\textbf{94.2 \pm 0.9} & \textbf{90.8 \pm 1.1} & \textbf{86.5 \pm 1.0} & \textbf{82.1 \pm 1.2} & \textbf{79.8 \pm 0.7} & \textbf{77.3 \pm 1.3} \\ 
    \bottomrule  
    
  \end{tabular}%
\end{table}

\paragraph{Influence of Depth.} We present the results for varying the number of layers for both SheafHyperGNN and SheafHyperGCN. The observations remain consistent for both the models using linear and non-linear sheaf Laplacian: the performance of traditional HyperGNN and HyperGCN generally decreases when increasing the depth of the model, a phenomenon known in the literature as over-smoothing. On the other hand, the results show that our generalized versions of the model, SheafHyperGNN and SheafHyperGCN respectively, do not suffer from this limitation, allowing for deeper architectures without sacrificing performance.

\begin{figure}[h!]
    \centering
    \includegraphics[scale=0.1]{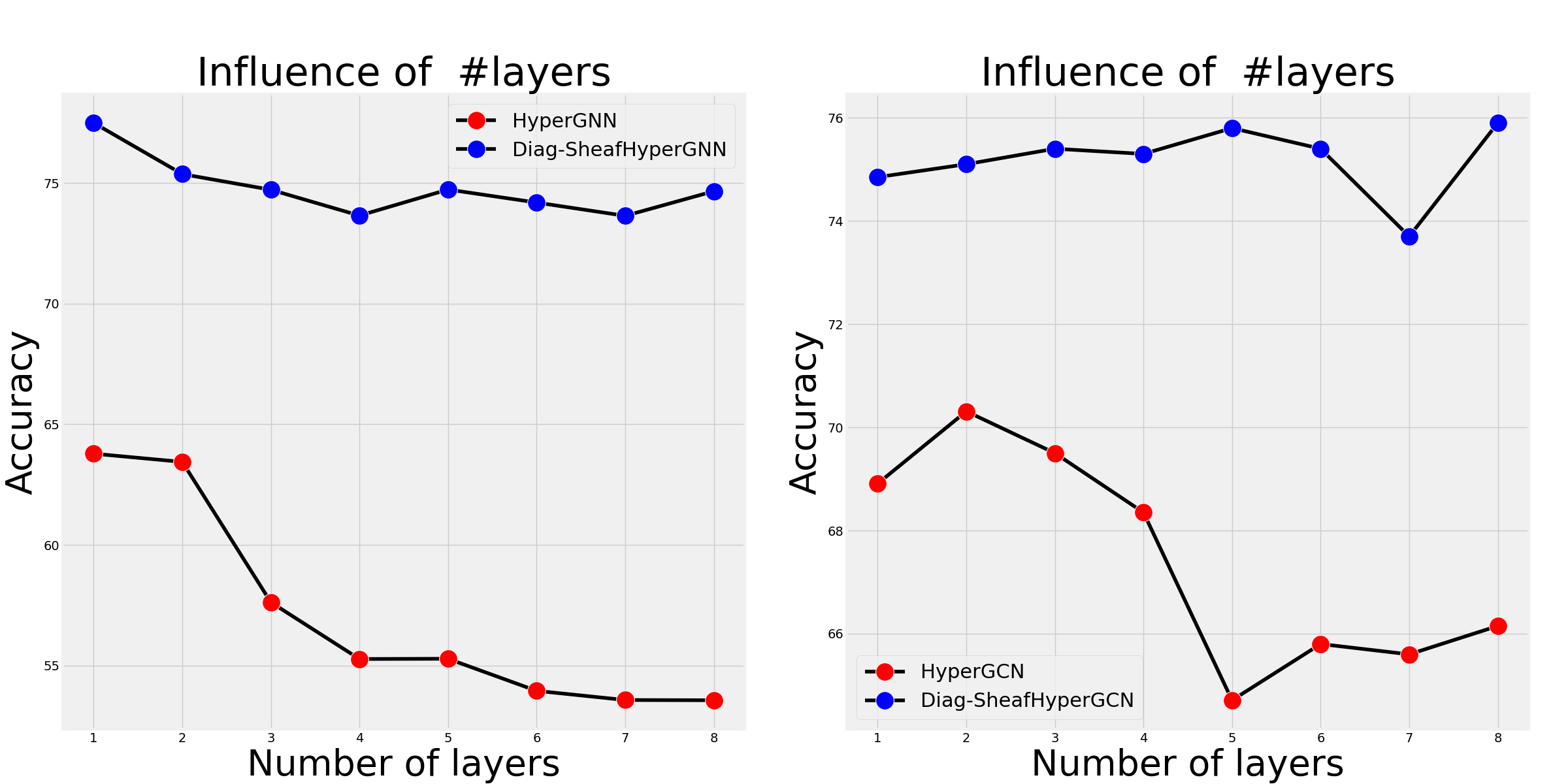}
    \caption{\textbf{Impact of Depth} for  both the linear (\textbf{left}) and the non-linear (\textbf{right}) sheaf hypergraph model, evaluated on the heterophilic dataset ($\alpha~=~7$). Both SheafHyperGNN and SheafHyperGCN's performance are generally unaffected by increasing depth, while the conventional HyperGNN and HyperGCN's performance tend to decrease when more than 2 layers are used.}
    \label{fig:main_fig}
\end{figure}

\end{document}